\documentclass[10pt, letterpaper]{article}
\usepackage{arxiv}

\title{\LARGE  Learning an Optimal Assortment Policy under Observational Data}

\author{
    Yuxuan Han\thanks{Stern School of Business, New York University. Email: \texttt{\{yh6061, zzhou\}@stern.nyu.edu}}\qquad Han Zhong\thanks{Center for Data Science, Peking University. Email: 
 \texttt{hanzhong@stu.pku.edu.cn}}\qquad Miao Lu\thanks{Department of Management Science \& Engineering, Stanford University. Email: \texttt{\{miaolu, jose.blanchet\}@stanford.edu}}\qquad Jose Blanchet\footnotemark[3]\qquad Zhengyuan Zhou\footnotemark[1]
}

\date{\small{\today}}

\begin{document}


\maketitle

\begin{abstract}
    We study the fundamental problem of offline assortment optimization under the Multinomial Logit (MNL) model, where sellers must determine the optimal subset of the products to offer based solely on historical customer choice data. 
    While most existing approaches to learning-based assortment optimization focus on the online learning of the optimal assortment through repeated interactions with customers, such exploration can be costly or even impractical in many real-world settings. 
    In this paper, we consider the offline learning paradigm and investigate the minimal data requirements for efficient offline assortment optimization.
    To this end, we introduce Pessimistic Rank-Breaking (PRB), an algorithm that combines rank-breaking with pessimistic estimation. 
    We prove that PRB is nearly minimax optimal by establishing the tight suboptimality upper bound and a nearly matching lower bound. This further shows that ``optimal item coverage'' --- where each item in the optimal assortment appears sufficiently often in the historical data --- is both sufficient and necessary for efficient offline learning. 
    This significantly relaxes the previous requirement of observing the complete optimal assortment in the data. Our results provide fundamental insights into the data requirements for offline assortment optimization under the MNL model.
\end{abstract}



\section{Introduction}\label{sec:Intro}

Assortment optimization is a fundamental problem in retail operations and e-commerce, where the seller must decide which subset of products to offer to maximize expected revenues. 
This decision is particularly challenging due to the complex relationships between customer preferences, product substitutions, and revenue implications  \citep{luce1959individual, train2009discrete, mcfadden2000mixed, asuncion2007uci, daganzo2014multinomial, davis2014assortment, alptekinouglu2016exponomial,blanchet2016markov,berbeglia2022comparative}.
The multinomial logit (MNL) model \citep{mcfadden1977modelling} has emerged as a popular framework for capturing customer behavior, providing a tractable approach to modeling how customers select among available products. 
The related assortment optimization problems based on the MNL model has been widely studied in the known model parameter setting over the past decades \citep{talluri2004revenue,davis2013assortment,avadhanula2016tightness,rusmevichientong2010dynamic}. However, as businesses increasingly rely on data-driven decision making, developing effective methods for data-driven assortment optimization under the MNL model has become even more crucial for maintaining competitiveness in both traditional retail and digital marketplaces.

Most existing learning-based algorithms \citep{saure2013optimal,agrawal2017thompson,agrawal2019mnl,saha2024stop} are designed for the online learning setup, where sellers learn through repeated interactions with the customers. 
However,  there are several critical limitations to this approach that motivate the need for a new learning paradigm for assortment optimization. 
First, online exploration can be prohibitively expensive or impractical in contexts involving high-stakes decisions, such as luxury goods retail or healthcare services, where experimental trials could lead to significant financial losses or could impact patient outcomes. 
Second, companies often cannot obtain immediate customer feedback in real-world business scenarios and update their strategies after each interaction. 
Third, many businesses already possess extensive historical customer data, making it redundant and inefficient to engage in time-consuming online learning iterations. In these situations, offline learning provides an effective solution, allowing sellers to make decisions based solely on existing data without requiring any online interactions. 
To our best knowledge, \citet{dong2023pasta} appears to be the only work that studies this problem and provides a provably efficient algorithm for solving it when the offline data has good \emph{optimal assortment coverage}.
Such a data coverage condition means that the offline dataset must contain sufficient information about the optimal assortment as a whole, that is, the offline data should contain enough information on which product a customer would choose provided the exact optimal assortment.
However, this assumption appears to be quite strong, since the number of possible assortments grows exponentially with the number of products, making it difficult to collect enough data for any particular assortment. Therefore, in this paper, we aim to address the following fundamental question:

\begin{center}
\textit{What is the minimal data condition that permits efficient offline assortment optimization \\ under the MNL model?}
\end{center}

We investigate this problem in two settings: (i) uniform reward setting, where all products have identical rewards, and (ii) non-uniform reward setting, where products have different rewards. We find that \emph{optimal item coverage}, where the offline data can provide sufficient information on each single item in the optimal assortment, is indeed the minimal condition for sample-efficient offline assortment optimization.
We demonstrate both the sufficiency and necessity of such a data condition. 

\paragraph{Sufficiency.} We develop Pessimistic Rank-Breaking (PRB), an algorithm that innovatively combines the rank-breaking technique \citep{saha2019active,saha2024stop} and the principle of pessimistic estimation in the face of uncertainty to solve offline assortment optimization with data satisfying optimal item coverage. 
Our theoretical analysis shows that the suboptimality gap of PRB in terms of the expected revenue is upper bounded by the following,
\begin{align*}
\begin{cases}
    \widetilde{\cO}\left(\sqrt{\frac{K}{\min_{i\in S^\star} n_i}}\right) & \quad \text{uniform reward setting}, \\
    \widetilde{\cO}\left(\frac{K}{\sqrt{\max_{i \in S^\star} n_i}}\right) & \quad \text{non-uniform reward setting},
\end{cases}
\end{align*}
where $K$ represents the size constraint for the assortment, $S^\star$ is the optimal assortment, $n_i$ represents the number of times for item $i$ in the optimal assortment to appear in the historical data, and $\widetilde{\cO}$ omits constants and logarithmic factors.  This result is significant because PRB only requires \emph{optimal item coverage} --- that each item in the optimal assortment has appeared individually in some historical assortment --- rather than requiring observations of the complete optimal assortment itself, as in \citet{dong2023pasta}.

\begin{table}
    \centering
    \begin{tabular}{|c|c|c|c|}
        \hline
         & Setting & Upper Bound & Lower Bound \\
        \hline
        \cite{dong2023pasta} & Non-uniform Reward & $\widetilde{\cO}\left(\sqrt{\frac{d}{n_{S^\star} \min_{i\in [N],\lvert S \rvert = K} p_i(S)}}\right)$ & --- \\
        \hline
         \cellcolor{blue!15}
 & \cellcolor{blue!15} Non-uniform Reward &  \cellcolor{blue!15} $\widetilde{\cO}\left(\frac{K}{\sqrt{\min_{i \in S^\star} n_i}}\right)$ & \cellcolor{blue!15} $\Omega\left(\frac{K}{\sqrt{\min_{i \in S^\star} n_i}}\right)$ \\
        \cline{2-4}
         \multirow{-2}{*}{\cellcolor{blue!15}Our Work} & \cellcolor{blue!15} Uniform Reward & \cellcolor{blue!15} $\widetilde{\cO}\left(\sqrt{\frac{K}{\min_{i\in S^\star} n_i}}\right)$  & \cellcolor{blue!15} $\Omega\left(\sqrt{\frac{K}{\min_{i\in S^\star} n_i}}\right)$ \\
        \hline
    \end{tabular}
    \caption{Comparison of offline assortment optimization results. 
    Here $S^{\star}$ denotes the optimal assortment, $n_i$ denotes the number of times that item $i$ in the optimal assortment appears in the offline data, $n_{S^{\star}}$ denotes the number of times that the optimal assortment appears in the offline data, $K$ denotes the size constraint of the assortment, and $p_i(S)$ denotes the choice probability of item $i$ given assortment $S$.
    We remark that the algorithm proposed in \citet{dong2023pasta} is applicable to the general $d$-dimensional linear MNL setting, where $d = N$ when it corresponds to our setting.}
    \label{tab:comparison}
\end{table}

\paragraph{Necessity.} We further establish the minimax lower bound of the suboptimality gap for the offline assortment optimization problem under the MNL model as following,
\begin{align*}
\begin{cases}
    \Omega\left(\sqrt{\frac{K}{\min_{i\in S^\star} n_i}}\right) & \quad \text{uniform reward setting}, \\
    \Omega\left(\frac{K}{\sqrt{\min_{i \in S^\star} n_i}}\right) & \quad \text{non-uniform reward setting},
\end{cases}
\end{align*}
which demonstrates that optimal item coverage is an essential requirement that cannot be avoided. 
Together with the upper bound of PRB, we conclude that PRB is nearly minimax optimal up to logarithmic factors. 

Based on these results, we reveal another interesting finding: there exists a \textbf{$\sqrt{K}$-statistical gap} between uniform and non-uniform reward settings. 
In summary, we have provided a complete characterization of the statistical limit of offline assortment optimization under the MNL model, and presented a nearly minimax optimal algorithm, PRB, for this problem.  In addition, our results provide three critical managerial insights:

\begin{enumerate}
\item We characterize in terms of a \textit{simple and interpretable} metric \textit{the key characteristic} (called the \textbf{item coverage number}) that defines whether existing historical data is \textit{necessary and sufficient} for effective assortment optimization. Specifically, the item coverage number reflects the amount of information about each item in the optimal assortment that is contained within the offline dataset, capturing the essential data requirement for efficient offline assortment learning (see Table~\ref{tab:comparison}).

\item We provide a novel algorithm, called PRB, which does not require as input the item coverage number and yet achieves near-optimal performance (in terms of  sub-optimality gap) by matching the theoretical minimax lower bounds for both uniform and non-uniform reward settings. PRB therefore guarantees that retailers can achieve the best possible results within the statistical limits of their available data.

    \item Our research demonstrates that retailers only need to ensure each potentially optimal item appears individually within their historical data, eliminating the need to observe complete optimal assortment combinations. This finding dramatically simplifies data collection requirements, as it removes the impractical burden of gathering data on all possible assortment permutations, which grows exponentially with the number of products.

\end{enumerate}

\section{Related Work}

\paragraph{Learning in Assortment Optimization.} The data-driven dynamic assortment optimization problem under the MNL choice model, where the seller does not know the problem parameters but can interact with customers repeatedly over $T$ rounds, has been extensively studied in the literature \citep{caro2007dynamic,rusmevichientong2010dynamic,saure2013optimal,agrawal2017thompson,agrawal2019mnl,chen2021optimal,saha2024stop}. 
Notably, for the problem with cardinality constraints, the works of \citet{agrawal2017thompson,agrawal2019mnl} and \citet{chen2018note} have established the $\widetilde{\Theta}(\sqrt{NT})$ minimax optimal regret bound. 
Beyond the basic MNL setting, learning problems involving additional constraints \citep{cheung2017assortment,aznag2021mnl,chen2024re} or under more complex choice models \citep{ou2018multinomial,oh2021multinomial,perivier2022dynamic,lee2024nearly,chen2021dynamic,li2022onlineassortment,zhang2024online} have also been explored during the past few years.

However, most of these results focus on the dynamic learning setting, with \citet{dong2023pasta} as an exception. Specifically, the work by \citet{dong2023pasta} presents the only study on offline assortment optimization problems, where their algorithm relies on the optimal assortment coverage condition on data. 
In contrast, our algorithm requires only a much weaker optimal item coverage, and we prove this condition to be minimal.

\paragraph{Offline Decision-Making and Pessimism Principle.} 
While optimism in the face of uncertainty is a well-established principle to design algorithms for various online learning problems, including the online assortment optimization problem, offline learning presents a different paradigm where pessimistic or conservative methods \citep{yu2020mopo,kumar2020conservative} have shown more effectiveness in terms of data efficiency and requirement. 
A line of recent works \citep{jin2021pessimism,rashidinejad2021bridging,xie2021bellman,uehara2021pessimistic,zhong2022pessimistic,zhan2022offline,liu2022welfare,shi2022pessimistic,lu2023pessimism,xiong2023nearly,rashidinejad2023optimal, blanchet2024double} theoretically demonstrates that in offline data-driven decision making, pessimistic algorithms achieve provable efficiency while only requiring good coverage of (the trajectories induced by) the optimal decision policy. 
Unfortunately, the settings addressed in all these results are not applicable to the MNL choice feedback in assortment optimization problems we consider here. 

\section{Problem Formulation}\label{sec:Method}

\paragraph{Assortment Optimization with Multinomial Logit Choice Model.} 
We study the assortment optimization problem, which models the interaction between a \emph{seller} and a \emph{customer}. Let $\{1, \ldots, N\}$ denote the set of $N$ available products/items. An assortment $S \subseteq \{1, \ldots, N\}$ represents the subset of products that the seller offers to the customer. When presented with assortment $S$, the customer chooses a product from the choice set $S_+ = \{0\} \cup S$, where $\{0\}$ represents the no-purchase option. In the multinomial logit (MNL) model, when a customer encounters assortment $S$, the probability that they choose product $i$ is given by
$$p_i(S):=\mathbb{P}\left(c_t=i \mid S_t=S\right)= \begin{cases}\frac{v_i}{1+\sum_{j \in S} v_j}, & \text { if } i \in S_+,  \\ 0, & \text { otherwise, }\end{cases}$$
where $v_i \geq 0$ represents the attraction parameter for product $i$. Without loss of generality, we set $v_0 = 1$ and assume $ \max_i v_i\leq V$ for some $V>0$. Each item $i$ generates a revenue $r_i \in [0,1]$ when purchased, while the no-purchase option generates no revenue: $r_0 = 0$. The learning objective is to maximize the expected revenue from the selected assortment, defined as
$$
R(S, \bm v): = \sum_{i\in S} r_i p_i(S) = \sum_{i\in S} \frac{r_i v_i}{1+\sum_{j\in S} v_j}.
$$
Following standard conventions, we analyze assortments $S$ where $|S| \leq K$, a constraint that reflects most real-world applications. 
When $K = N$, it recovers the unconstrained case, permitting assortments of any size.

\paragraph{Offline Learning for Assortment Optimization.} 
In the offline learning setting, the seller does not know the underlying parameter $\bm{v}$ but has access to a pre-collected dataset $\{i_k,S_k\}_{k = 1}^n$ consisting of choice-assortment pairs. 
While the observed data may have a dependency structure, we assume conditional independence: for each $k \in [n]$, given $S_k$, the choice $i_k$ is sampled from the MNL choice model independently of all previous choices $\{i_m\}_{m \neq k}$, this condition is naturally met by datasets collected in an online decision-making process. 
The learning objective is the sub-optimality gap of the expected revenue, defined as,
\begin{align*}
\SubOpt(S;\bm v) :=  R(S^\star;\bm v) - R(S;\bm v),
\end{align*}
which measures how far an assortment $S$ deviates from optimal performance by calculating the difference between its expected return and that of the optimal assortment. We aim to develop efficient algorithms that use offline data to identify an assortment with a minimal sub-optimality gap.

\section{Algorithms and Theoretical Results}

\subsection{Pessimistic Rank-Breaking}

To address the offline assortment problem, we introduce an algorithm called Pessimistic Rank-Breaking (PRB), with its detailed implementation shown in Algorithm~\ref{alg-rank-breaking}. PRB combines two key concepts: \textit{rank-breaking based estimation} and \textit{pessimistic assortment optimization.}

\paragraph{Rank-Breaking Based Estimation.} 
Rank-breaking is a technique that can convert ranking data into independent pairwise comparisons, which can then be used to construct pairwise estimators for the unknown parameters \citep{saha2019active,saha2024stop}. 
In the context of assortment optimization under the MNL model, we develop estimators for the pairwise choice probabilities $\mathbb{P}(j \mid {0, j}) = v_j/(1 + v_j)$ for different $j\in[N]$. 
The empirical estimation is formulated as
\begin{align*}
    \widehat{p}_j = \frac{\tau_j}{\tau_{j0}}, \quad \text{where} \quad \tau_{j0} = \sum_{k=1}^n \mathbf{1}\big\{i_k \in \{0, j\}, j \in S_k\big\}, \quad \tau_j = \sum_{k=1}^n \mathbf{1}\{i_k = j\},
\end{align*}
which then can imply an estimator of $v$ based on the relation $v_j = p_j/(1-p_j)$. 
The rank-breaking method achieves optimal sample efficiency with minimal requirement on the observational assortment dataset by addressing two key limitations of other popular approaches for estimating the underlying attraction parameter. Specifically, it avoids reliance on $\min_{i \in S} p_i(S)$, which is required by the analysis of MNL maximum likelihood estimator \citep{oh2021multinomial,perivier2022dynamic,lee2024nearly}.
Also, rank-breaking does not require repeated observations of sampled assortment sequences, a requirement of the geometric variable-based estimator \citep{agrawal2017thompson,agrawal2019mnl}. These advantages are particularly important in offline learning settings, which allow us to only make a weak assumption on the observed assortment dataset.

\paragraph{Pessimistic Assortment Optimization.}
Based on the rank-breaking estimator, we construct a lower confidence bound (LCB) for the true pairwise choice probability $\mathbb{P}(j \mid {0, j})$ as,
\begin{align*}
    p_j^{\text{LCB}} = \left(\widehat{p}_j - \sqrt{\frac{2\widehat{p}_j(1-\widehat{p}_j)\log(1/\delta)}{\tau_{j0}}} - \frac{\log(1/\delta)}{\tau_{j0}}\right)_+.
\end{align*}
The error analysis of $\widehat{p}_j$, as detailed in Proposition~\ref{prop-vLCB}, shows that $p_j^{\text{LCB}}$ serves as a pessimistic estimator of $p_j$, i.e., $p_j^{\text{LCB}} \leq p_j$ with high probability. Consequently, the $p_j^{\text{LCB}}$-based estimator for $v_j$, defined as
\begin{align*}
    v_0^{\mathrm{LCB}}=1,\qquad v_j^{\text{LCB}} = p_j^{\text{LCB}} / (1 - p_j^{\text{LCB}}),\quad\forall j\in[N],
\end{align*}
is also a pessimistic estimator of $v_j$. 

Then the key observation  behind the design of our algorithm is that, when the learner performs revenue maximization using a sequence of pessimistic attraction parameter estimates, the sub-optimality gap between the resulting assortment and the true optimal assortment is determined solely by the estimation errors for items in $S^\star$. 
This observation paves the way for sample-efficient offline assortment learning under a coverage condition only restricted to items in $S^\star$, rather than requiring full coverage of all items. 
We formalize this observation in the following lemma, which is proved in Appendix~\ref{subsec: proof pessimistic}.

\begin{lemma}\label{lem-SubOptGap-via-pessmistic}
    Suppose there exists a sequence $\{\underline{v}_j \}_{j = 1}^N$ so that $v_j \geq \underline{v}_j,\forall j \in [N].$ Then it holds that for the revenue maximizer $\underline{S}:= \arg\max_{\lvert S \rvert \leq K} R(S;\underline{\bm v}),$ 
    \begin{align*}
        R(S^\star;\bm v) - R(\underline{S};\bm v) \leq \sum_{j \in S^\star} \frac{r_j(v_j - \underline{v}_j)}{1+\sum_{j\in S^\star} v_j}.
    \end{align*}
\end{lemma}

Therefore, based on the pessimistic attraction parameter estimates, PRB finally outputs $\widehat{S}$ that maximizes the expected revenue with respect to ${\bm v}^{\mathrm{LCB}} = \{v_j^{\mathrm{LCB}}\}_{j = 1}^N$,
\begin{align*}
    \widehat{S}:= \argmax_{\lvert S \rvert \leq K} R(S;{\bm v}^{\mathrm{LCB}}).
\end{align*}
While this step involves solving a maximization problem over an exponentially large set, it is well-studied in the assortment optimization literature and can be solved in polynomial time using several well-known algorithms \citep{rusmevichientong2010dynamic, davis2013assortment, avadhanula2016tightness}.

\begin{algorithm}[t]
\caption{Pessimistic Rank-Breaking (PRB)} \label{alg-rank-breaking}
\begin{algorithmic}[1]
\STATE \textbf{Inputs:} Offline dataset $\{i_k, S_k\}_{k = 1}^n$, failure probability $\delta$
\FOR{$j \in [N]$}
    \STATE $\tau_{j0} \leftarrow \sum_{k=1}^n \mathbf{1}\{i_k \in \{0, j\}, j \in S_k\}$, $\tau_j \leftarrow \sum_{k=1}^n \mathbf{1}\{i_k = j\}$
    \STATE $\widehat{p}_j \leftarrow \tau_j / \tau_{j0}$ \textcolor{blue}{//Compute the empirical probability that item $j$ beats $0$}
    \STATE $p_j^{\text{LCB}} \leftarrow \left(\widehat{p}_j - \sqrt{\frac{2\widehat{p}_j(1-\widehat{p}_j)\log(1/\delta)}{\tau_{j0}}} - \frac{\log(1/\delta)}{\tau_{j0}}\right)_+$\\ \STATE $v_j^{\text{LCB}} \leftarrow p_j^{\text{LCB}} / (1 - p_j^{\text{LCB}})$ \textcolor{blue}{//Compute pessimistic values}
\ENDFOR
\STATE  Output $\widehat{S}:= \arg\max_{\lvert S \rvert \leq K} R(S;{\bm v}^{\mathrm{LCB}})$.
\end{algorithmic}
\end{algorithm}

\subsection{Theoretical Results}

We first present the theoretical guarantee of the PRB algorithm in a non-uniform reward setting, where each product yields a different reward value.

\begin{theorem}[Sub-Optimality Gap for Non-Uniform Reward Setting]\label{thm-upper-bound-non-uniform}  
Define $n_i: = \sum_{k = 1}^n \bm{1}\{i\in S_k\}$ and assume that $n_i\geq 256(1+KV)\log(N/\delta)$ for all $i\in S^\star$. Then, with probability at least $1-2\delta$, the output $\widehat{S}$ of Algorithm~\ref{alg-rank-breaking} satisfies that
\begin{align*}
        \mathrm{SubOpt}( \widehat{S}; \bm{v})  \le  2(1+V)K\cdot \sqrt{\frac{\log (N/\delta)}{\min_{i\in S^\star}n_i}} + \frac{48K^2(1+V)^2\log(N/\delta)}{\min_{i\in S^\star} n_i}.
\end{align*}
\end{theorem}

\begin{proof}[Proof of Theorem~\ref{thm-upper-bound-non-uniform}]
    Please see Appendix~\ref{subsec: proof upper bound non uniform} for a detailed proof.
\end{proof}

Theorem~\ref{thm-upper-bound-non-uniform} provides the finite-sample guarantee for the PRB algorithm. Notably, in the non-trivial setting where $(1+V)K\sqrt{\log(N/\delta)/\min_{i\in S^\star} n_i} =\cO(1),$ the second term is dominated by the first term, thus the sub-optimality gap converges in a rate $\widetilde{\cO}(K/\sqrt{\min_{i\in S^\star}n_i}) $. In contrast, the best-known previous result for offline assortment optimization, presented in \cite{dong2023pasta}, is given by \begin{align}\label{eq-pasta-bound}
   \SubOpt(\widehat{S};\bm v) =   \widetilde{\cO}\left(\sqrt{\frac{{N}}{n_{S^\star}\min_{i\in [N],\lvert S \rvert = K}  p_i(S)}}\right),
\end{align}
under our notation. This result involves the term $n_{S^\star}^{-1/2}$ with $n_{S^\star} := \sum_{k = 1}^n \bm{1}\{S_k = S^\star\}.$ It is straightforward to show that $\max_{i \in S^\star} n_i \geq n_{S^\star}$, and it is even possible for $\max_{i \in S^\star} n_i$ to be large while $n_{S^\star} = 0$. Our result demonstrates that, instead of requiring observations of exact optimal assortments, it suffices to observe assortments that cover each item in the optimal assortment to learn the optimal policy. Beyond the improved from $n_{S^\star}$ to $\min_{i\in S^\star} n_{i}$, the problem parameter $\sqrt{N/\min_{i\in [N],\lvert S\rvert = K} p_i(S)}$ involved in \eqref{eq-pasta-bound} can blow up to $+\infty$ as some $v_i$ approaches $0$. Even in the best case, where all $p_i(S) = 1/K$,this problem-dependent parameter scales as $\sqrt{NK}$, strictly larger than our result, which is solely linear in $K$.

A natural question then arises: can the offline data coverage condition be further relaxed, and can the dependency on problem parameters be further improved? In the following, we show that the dependence of sub-optimality gap in offline learning on the item-wise coverage number $\max_{i \in S^\star} n_i$ is inevitable and that our upper bound has an optimal dependency on this quantity. 

\begin{theorem}[Lower bound for Non-Uniform Reward Setting]\label{thm-lower-bound-non-uniform}
For any sufficiently large integers $N,K$ satisfying $N \geq 5K$, there exists a set of observed $K$-sized assortments $D_S = \{S_k\}_{k=1}^n \subset [N]^n$, a reward vector $\bm r\in [0,1]^N$ and a class of attraction parameters $\mathcal{V} \subset [0,1]^N$, so that for any algorithm $\mathcal{A}$ that takes as input $D_S$ and $$D_{\bm{v}} := \{i \in S : S \in D_S, i \sim \mathbb{P}(\cdot \mid S, \bm{v}) \text{ independently}\},$$ if we denote $S^\star_{\bm{v}}$ the optimal assortment under $\bm{v}$, then under the reward parameter $\bm r$:
\begin{align*}
\max_{\bm{v} \in \mathcal{V}} \mathbb{E}_{\bm{v}} \left[  \frac{ R(S_{\bm{v}}^\star; \bm{v}) - R(\mathcal{A}(D_{\bm{v}}, D_S); \bm{v}) }{K/ \sqrt{\min_{i \in S^\star_{\bm{v}}} n_i}} \right] \ge \Omega(1).
\end{align*}
\end{theorem}

\begin{proof}[Proof of Theorem~\ref{thm-lower-bound-non-uniform}]
    Please see Section~\ref{sec-lower-bound-proof} for a detailed proof.
\end{proof}
Theorem~\ref{thm-upper-bound-non-uniform} and Theorem~\ref{thm-lower-bound-non-uniform} together implies a $\widetilde{\Theta}({K}/{\sqrt{\min_{i\in S^\star} n_i}})$ minimax rate on offline assortment learning. 
In particular, this result shows that the information-theoretic limit scales linearly with the assortment size $K$ and the item-wise coverage condition is inevitable.

In addition to the general non-uniform formulation discussed above, another widely studied scenario in assortment optimization is the uniform-reward setting \citep{oh2021multinomial, perivier2022dynamic, lee2024nearly}, where a common reward $r \in [0,1]$ exists such that $r_i \equiv r$ for all $i \in [N]$.  This special case yields a better $K$-dependency result compared to that established in Theorem~\ref{thm-upper-bound-non-uniform}.

\begin{theorem}[Sub-Optimality Gap for Uniform Reward Setting]\label{thm-upper-bound-uniform}
Define $n_i: = \sum_{k = 1}^n \bm{1}\{i\in S_k\}$ and assume that $n_i\geq 256(1+KV)\log(N/\delta)$ for all $i\in S^\star$. 
Then, with probability at least $1-2\delta$, the output of Algorithm~\ref{alg-rank-breaking} satisfies that
\begin{align*}
        \mathrm{SubOpt}( \widehat{S}; \bm{v})  \le  2(1+V)\cdot \sqrt{\frac{K\log (N/\delta)}{\min_{i\in S^\star}n_i}} + \frac{48K(1+V)\log(N/\delta)}{\min_{i\in S^\star} n_i}.
\end{align*}
\end{theorem}

\begin{proof}[Proof of Theorem~\ref{thm-lower-bound-uniform}]
    Please see Appendix~\ref{subsec: proof upper bound uniform} for a detailed proof.
\end{proof}

Due to the additional uniform reward assumption, Theorem~\ref{thm-upper-bound-uniform} provides an improvement of order $\sqrt{K}$ over Theorem~\ref{thm-upper-bound-non-uniform}, disregarding lower-order terms. 
This upper bound is proven to be tight due to a corresponding lower bound that nearly matches it.

\begin{theorem}[Lower bound for Uniform Reward Setting]\label{thm-lower-bound-uniform}
For any sufficiently large integers $N,K$ such that $N\geq 5K$, when the reward vector $\bm{r} = \bm{1}_{N}$, there exists a set of observed $K$-sized assortments $D_S = \{S_k\}_{k=1}^n\subset [N]^n$ and a set of problem parameters $\mathcal{V}\subset [0,1]^N$. So that for any algorithm $\mathcal{A}$ that takes as input $D_S$ and $$D_{\bm{v}} := \{i \in S : S \in D_S, i \sim \mathbb{P}(\cdot \mid S, \bm{v}) \text{ independently}\},$$ if we denote $S^\star_{\bm{v}}$ the optimal assortment under $\bm{v}$, then under the reward parameter $\bm r$:
\begin{align*}
\max_{\bm{v} \in \mathcal{V}} \mathbb{E}_{\bm{v}} \left[  \frac{ R(S_{\bm{v}}^\star; \bm{v}) - R(\mathcal{A}(D_{\bm{v}}, D_S); \bm{v}) }{\sqrt{K/\min_{i \in S^\star_{\bm{v}}} n_i}} \right] \ge \Omega(1).
\end{align*}
\end{theorem}

\begin{proof}[Proof of Theorem~\ref{thm-lower-bound-uniform}]
    Please see Appendix~\ref{appendix-proof-lower-bound-uniform} for a detailed proof.
\end{proof}

Based on Theorems~\ref{thm-upper-bound-non-uniform}-\ref{thm-lower-bound-uniform}, we have established statistical limits of $\Theta(K/\sqrt{\min_{i \in S^*} n_i})$ for the non-uniform reward setting and $\Theta(\sqrt{K/\min_{i \in S^*} n_i})$ for the uniform reward setting. An interesting phenomenon is the presence of a $\sqrt{K}$-statistical gap between these two settings. This gap is unique to the offline setting and does not appear in the online setting, where both uniform and non-uniform reward settings have the regret upper and lower bounds of $\sqrt{NT}$, where $N$ represents the number of products and $T$ represents the number of online iterations.

\section{Proof of Lower Bound Results}\label{sec-lower-bound-proof}

In this section, we provide the proof of our lower bound results. 
Here we mainly focus on the non-uniform reward setting (Theorem~\ref{thm-lower-bound-non-uniform}) and comment on the differences in the proofs for the uniform reward setting, leaving the details of the latter to Appendix~\ref{appendix-proof-lower-bound-uniform}.
Lemmas in this section are proved in Appendix~\ref{sec: proof lemma lower bound}.

For sufficiently large cardinality $K$, number of items $N$, and the effective sample size $\nmin$ with $N \geq 5K, \sqrt{\nmin}\geq K$, we construct a class of problem parameters $\mathcal{V} \subset \mathbb{R}_{+}^N \times \mathbb{R}^N_+$. Each $(\bm r,\bm{v}) \in \mathcal{V}$ pair corresponds to a combination of revenues and attraction values of $[N]$ items, we also denote $S^\star_{\bm r, \bm v}$ the optimal assortment under parameter $(\bm r, \bm v)$ and $\mathbb{Q}_{\bm v}(S)$ the corresponding distribution that generates the customer choice after given assortment $S$. Our goal is to show that there exists an observable assortment set $\{S_1,\dots,S_n\}$ so that 
\begin{enumerate}
    \item For any problem parameter $(\bm r, \bm v) \in \mathcal{V}$, it holds that $\min_{i\in S^\star_{\bm r, \bm v}}n_i \geq \nmin$.
    \item For any policy $\pi$ computing the approximate optimal assortment with input $D:=\{(i_1,S_1),\dots,(i_n,S_n)\},$ $i_k \sim \mathbb{Q}_{\bm v}(S_k)$ independently, it holds that \begin{align*}
    \max_{(\bm v, \bm r)\in \mathcal{V}}\E\big[ R(S^\star_{\bm r,\bm v};\bm r, \bm v) -   R(\pi(D);\bm r, \bm v) \big] \gtrsim \frac{K}{\sqrt{\nmin}},
\end{align*}
where the expectation is taken over the distribution of $D.$
\end{enumerate}

\subsection{Construction of Hard Problem Instances}\label{sec-construction}

In the following, we first introduce our construction of $\mathcal{V}$ and $\{S_1,\dots,S_n\}$, then show that the revenue gap lower bound of assortment learning problem can be reduced to lower bound the error rate of a hypothesis testing problem.

\paragraph{Construction of $\mathcal{V}$.} To explain our construction on $\mathcal{V}$, we divide the set of items $[N]$ into three distinct subsets:
\begin{enumerate}
    \item \textbf{Optimal items} $\mathcal{N}_{\text{opt}}$:  A $K$-sized subset of $[4K]$ such that each item in this subset has attraction value of $1/K + \epsilon$ and revenue of $1$, with $\epsilon <\frac{1}{4K}$ to be determined later.
    \item \textbf{Competitive items} $\mathcal{N}_{c} := [4K]\setminus \Nopt$: each item in this subset has an attraction value of $1/K$ and a revenue of $1$.
    \item \textbf{Significantly sub-optimal items} $\mathcal{N}_{0} := [N] \setminus [4K]$: each item in this subset has an attraction value of $1$ and a revenue of $0$.
\end{enumerate}
In the above construction, the optimal assortment is precisely the set consisting of all items in $\Nopt$. The set $\Nc$ contains items that are sub-optimal but have only a small gap of $\epsilon$ in attraction values compared to those in $\Nopt$. Lastly, $\Nsubopt$ contains items that have much larger attractive values than other items but are definitively sub-optimal due to their zero revenue. Each partition above corresponds to a pair of problem parameters $\bm v, \bm r$ with \begin{align*}
    \bm{v}_{\mathcal{N}_{opt}} = 1/K+\epsilon,   \bm{v}_{\mathcal{N}_{c}} = 1/K, \bm{v}_{\mathcal{N}_0} = 1,\quad \text{ and }\quad \bm{r}_{\mathcal{N}_{opt}} = 1,   \bm{r}_{\mathcal{N}_{c}} = 1, \bm{r}_{\mathcal{N}_0} = 0. 
\end{align*}

To define $\mathcal{V},$ we introduce the following packing number result under the distance 
\begin{align}\label{eq-dist-Delta}
\Delta(S,S'):= \lvert (S \cup S') \setminus (S\cap S') \rvert
\end{align}
defined over $\mathcal{S}_K:$
\begin{lemma}\label{lem-packing-number}
    There exists some absolute constants $C_0,C_1>0$ so that when $K>C_0$, there exists a collection $\mathcal{F}$ of $K$-sized subsets of $[4K]$ so that $\Delta(S,S') \geq \lfloor K/4 \rfloor ,  \forall S, S' \in \mathcal{F}$ and $\log \lvert \mathcal{F} \rvert \geq C_1 K.$
\end{lemma}

With the above $K/4$-packing collection $\mathcal{F}$ of $[4K]$, we define $\mathcal{V}$ as the set consisting of all $(\bm{v}, \bm{r})$ pairs induced by all possible partitions described above with $\Nopt \in \mathcal{F}$. Note that the partition is fully determined by the choice of $\mathcal{N}_\text{opt}$, and therefore, the size of $\mathcal{V}$ satisfies $\log\lvert \mathcal{V} \rvert =\log \lvert \mathcal{F} \rvert \geq C_1K$. Since there is a one-to-one correspondence between $(\bm r, \bm{v}) \in \mathcal{V}$ and the partition, we use the notation $\Nopt(\bm r, \bm{v})$, $\Nc(\bm r, \bm{v})$, and $\Nsubopt(\bm r, \bm{v})$ for $(\bm r, \bm{v}) \in \mathcal{V}$ to denote the optimal, competitive, and significantly sub-optimal item sets, respectively, under the parameters $\bm r, \bm{v}$.

\paragraph{Construction of $\{S_1,\dots,S_n\}$.} We select $n = 4K\nmin$ and let $$S_i = \{ \lceil i/\nmin\rceil  , 4K+1,\dots,5K-1 \},\quad \forall 1\leq i \leq n.$$
By our construction, for every $(\bm r, \bm v) \in \mathcal{V}$, each item in $\Nopt(\bm r, \bm v)\cup \Nc(\bm r, \bm v)$ is covered by $\{S_1,\dots,S_n\}$ exactly $\nmin$ times. In each observation, only one of these items is accompanied by $K-1$ significantly sub-optimal items. In the following proof, we  index each $S_{(\ell-1) \nmin + j}$ differently as $S_{j}^{(\ell)}$ for $\ell \in [4K]$ and $j \in [\nmin]$ to simplify our notation and emphasize that it is the $j$-th assortment containing item $\ell$.

\subsection{Applying Fano's Lemma}

Firstly, by our construction , we have the following lower bound result regarding the sub-optimality gap via cardinality-based set distance $\Delta(\cdot,\cdot),$ defined in \eqref{eq-dist-Delta}
\begin{lemma}\label{lem-lb-via-distance} For any $(\bm r, \bm v)\in \mathcal{V}$ and $S\subset \mathcal{S}_K,$ it holds that \begin{align*}
    \SubOpt(S;\bm r, \bm v) \geq  \frac{\epsilon \Delta(\Nopt(\bm r, \bm v) ,S)}{18}.
\end{align*}
\end{lemma}
Thus then our goal can be reduced to show \begin{align}\label{eq-reduce-to-distance-lb}
    \min_{\pi}\max_{(\bm r, \bm v) \in \mathcal{V}} \E[ \Delta (\Nopt(\bm r, \bm v),\pi (D))] \gtrsim \frac{K}{\epsilon\sqrt{\nmin}}.
\end{align}
Now we are ready to prove~\eqref{eq-reduce-to-distance-lb} via the Fano's lemma, a standard tool to study the information-theoretic limits \citep{yu1997assouad,tsybakov2009nonparametric,le2012asymptotic}:

\begin{lemma}\label{lem-fano} Let $ \Gamma:=  \{\theta^1,\dots, \theta^M\}\subset \Theta$ be a $2\delta$-separated set under some metric $\rho$ over $\Theta,$ each associated with a distribution $\mathbb{P}_{\theta}$ over some set $\mathcal{X},$ then it holds that \begin{align*}
    \min_{\pi} \max_{k\in [M]}\E_{D\sim P_{\theta_k}}[\rho(\pi(D),\theta_k) ] \geq \delta \cdot\left(1 - \frac{\frac{1}{M^2}\sum_{i,j = 1}^M D(\mathbb{P}_{\theta_i}\lVert \mathbb{P}_{\theta_j})  +\log 2}{\log M}\right).
\end{align*}
\end{lemma}

If we let $\Theta = \mathcal{S}_K$, $\Gamma = \{\Nopt(\bm r, \bm v): (\bm r, \bm v) \in \mathcal{V}\}$, $\rho(\cdot,\cdot) = \Delta(\cdot,\cdot)$, then by definition $\Gamma$ is $2\delta_0$-separated with $\delta_0 = K/4$. Now denote 
$$\mathbb{P}_{\bm r, \bm v}: = \prod_{\ell=1}^{4K}\prod_{j = 1}^{\nmin} \mathbb{Q}_{\bm v}(S_{j}^{(\ell)}), $$ 
we have by $\log M = \log \lvert \mathcal{V}\rvert \geq C_1K,$ Lemma~\ref{lem-fano} implies that when $K\geq C_2:=\max\{C_0, C_1^{-1}2\log 2\} $ \begin{align*}
        \min_{\pi}\max_{(\bm r, \bm v) \in \mathcal{V}} \E[ \Delta (\Nopt(\bm r, \bm v),\pi (D))] \gtrsim  K\left(\frac{1}{2} - \frac{\sum_{(\bm r, \bm v)\in \mathcal{V}}\sum_{(\bm r', \bm v')\in \mathcal{V}} D(\mathbb{P}_{\bm r, \bm v} \lVert \mathbb{P}_{\bm r', \bm v'}) }{C_1 K\lvert \mathcal V \rvert^2} \right).
\end{align*}
And it remains to provide valid upper bound for the summation of KL divergence terms.

\subsection{Upper Bounding the KL-Divergence}

For any pairs $(\bm r, \bm v),(\bm r', \bm v') \in \mathcal{V},$ we have 

\begin{align*}
    D(\mathbb{P}_{\bm r, \bm v} \lVert \mathbb{P}_{\bm r', \bm v'}) &= D\left( \prod_{\ell=1}^{4K}\prod_{j = 1}^{\nmin} \mathbb{Q}_{\bm v}(S_{j}^{(\ell)}) \bigg\lVert \prod_{\ell=1}^{4K}\prod_{j = 1}^{\nmin} \mathbb{Q}_{\bm v'}(S_{j}^{(\ell)})  \right)= n_{\min}\sum_{\ell = 1}^{4K} D(\mathbb{Q}_{\bm v} (S^{(\ell)}_1) \rVert \mathbb{Q}_{\bm v'} (S^{(\ell)}_1)).
\end{align*}
and the following result on each $D(\mathbb{Q}_{\bm v} (S^{(\ell)}_1) \rVert \mathbb{Q}_{\bm v'} (S^{(\ell)}_1))$:
\begin{lemma}\label{lem-KL-new-upper-bound-non-uniform}
    For every $\ell,$ it holds that \begin{align*}
        D(\mathbb{Q}_{\bm v} (S^{(\ell)}_1) \rVert \mathbb{Q}_{\bm v'} (S^{(\ell)}_1)) \leq 
\begin{cases}
    5\epsilon^2,  & \text{if }  \ell \in (\Nopt(\bm r, \bm v) \cup \Nopt(\bm r', \bm v'))\setminus(\Nopt(\bm r, \bm v)\cap \Nopt(\bm r', \bm v')), \\
    0,& \text{otherwise.}
\end{cases} 
    \end{align*}
\end{lemma}

Lemma~\ref{lem-KL-new-upper-bound-non-uniform} indicates that with the selection $\epsilon =\sqrt{\frac{C_1} {20\nmin}},$ we have $\max_{(\bm r, \bm v),(\bm r', \bm v')} D(\mathbb{P}_{\bm r, \bm v} \lVert \mathbb{P}_{\bm r', \bm v'})\leq C_1K/4,$ which then leads to

\begin{align*}
  K\left(\frac{1}{2} - \frac{\sum_{(\bm r, \bm v)\in \mathcal{V}}\sum_{(\bm r', \bm v')\in \mathcal{V}} D(\mathbb{P}_{\bm r, \bm v} \lVert \mathbb{P}_{\bm r', \bm v'}) }{C_1 K\lvert \mathcal V \rvert^2} \right)\geq \frac{K}{4},
\end{align*}
that finishes the proof of \eqref{eq-reduce-to-distance-lb}, which then implies Theorem~\ref{thm-lower-bound-non-uniform}.

\subsection{Additional Remarks and Discussion on Uniform-Reward Setting}
\paragraph{Comparison to the Lower Bound Proof in Online MNL Setting.} In the construction of the hard instance, our selection of problem parameters for the optimal item set and competitive item sets is inspired by \citet{chen2018note} for online MNL regret lower bounds. On the other hand, the construction of $\Nsubopt$ is new and specifically designed for the offline setting to achieve the sharp dependency on $K$. 

To explain the role of $\Nsubopt$, in our construction of observed assortment sets, every item in the optimal set or competitive set is observed alongside other items in $\Nsubopt$. Notably, items in $\Nsubopt$ are assigned large attraction values, making it rare to observe choices involving items from $\Nopt$ or $\Nc$ for any $S_i$. This rarity, in turn, increases the difficulty for the learner to distinguish the optimal items in $\Nopt$ from those in $\Nc$.

\paragraph{On the Proof of Uniform-Reward Setting.} Our framework developed in this section still works for the uniform-reward setting, where the revenue of all items are equal. However, several modifications to the instance construction and recalculation of the KL upper bounds in Lemma~\ref{lem-KL-new-upper-bound-non-uniform} are required. The main difference on the instance construction is that the $\bm r, \bm v$ assigned for items in $\Nsubopt$ are not valid. More precisely, when requiring all $r$ to be equal, the optimal assortment is now given by the items with the top-$K$ attraction values. Consequently, it is no longer possible to assign sub-optimal items higher attraction values than the optimal ones. In the construction of this setting, we set $\bm r_{\Nsubopt}$ to $1$ and $\bm v_{\Nsubopt}$ to $1/K$. This modification, in particular, substantially increases the probability of observing a choice for an item in $\Nopt \cup \Nc$ within the dataset, thereby making it easier to distinguish between items in $\Nopt$ and $\Nc$. This difference then allows for a $\sqrt{K}$ improvement in the uniform-reward setting over the non-uniform result. We leave the detailed proof to Appendix~\ref{appendix-proof-lower-bound-uniform}.

\section{Experiments}

In this section, we perform several numerical studies to illustrate the empirical performance of the proposed PRB algorithm.

\subsection{Baselines and Computational Cost Analysis}
 We
compare the \textbf{PRB} algorithm with the \textbf{PASTA} algorithm \citep{dong2023pasta} and a non-pessimistic maximum likelihood estimator (\textbf{MLE}) baseline. 
Since all methods rely on first estimating the attraction parameter $\widehat{v}$ and then taking the revenue maximizer under $\widehat{v},$ we only compare the complexity of obtaining the attraction parameter estimator $\widehat{v}$ here. For the PRB algorithm, it can be observed from the algorithm design that $\cO(n)$ additions of $K$-dimensional vectors are required to compute $\tau_{0j}$ and $\tau_j$ for all $j \in [N]$, followed by $\cO(N)$ operations to compute the final pessimistic value estimator. 

On the other hand, the MLE-based algorithm, whose value estimator is given by
\begin{align*}
\widehat{v} = \exp(\widehat{\theta}), \quad \widehat{\theta} = \text{argmax}_{\theta \in \mathbb{R}^d} \underbrace{\sum_{t = 1}^T \sum_{j \in (S_t)_+} \bm{1}\{i_t = j\} \log \left(\frac{e^{\theta_j}}{1 + \sum_{k \in S_t} e^{\theta_k}} \right)}_{:= \ell(\theta)},
\end{align*}
requires iterative optimization methods to maximize the objective function $\ell(\theta)$. Here, we compare with the first-order methods, which have the lowest computational cost per iteration. Specifically, at each iteration, computing the gradient incurs a cost of $\cO(nK)$, while updating the parameter $\theta$ requires $\cO(N)$ operations. Consequently, the total computational cost of PRB is comparable to the cost of a single iteration of the gradient-based optimization method used to solve MLE. Finally, the PASTA algorithm has even higher computational complexity compared with MLE-based algorithm since it requires the MLE result as a initial value estimate. In summary, the PRB algorithm is highly computationally efficient, requiring only as much cost as a single iteration of an MLE-based algorithm or PASTA. This efficiency allows PRB to scale seamlessly to large item size $N$ and sample size $n$.

\subsection{Sample Efficiency of PRB}

We examine the performance of different baselines using randomly generated synthetic datasets. We conduct experiments varying $\min_{i \in S^\star} n_i$ with a fixed $K$, and varying $K$ with a fixed $\min_{i \in S^\star} n_i^\star$, to show how the empirical performance depends on $\min_{i \in S^\star} n_i$ and $K$.

\subsubsection{Dependency on  $\min_{i\in S^\star}n_i$}\label{sec-experiment-partial-coverage}

\begin{figure}[htbp]
    \centering
    \subfloat[$\epsilon = 0.1/\sqrt{K}$]{
        \includegraphics[width=0.45\textwidth]{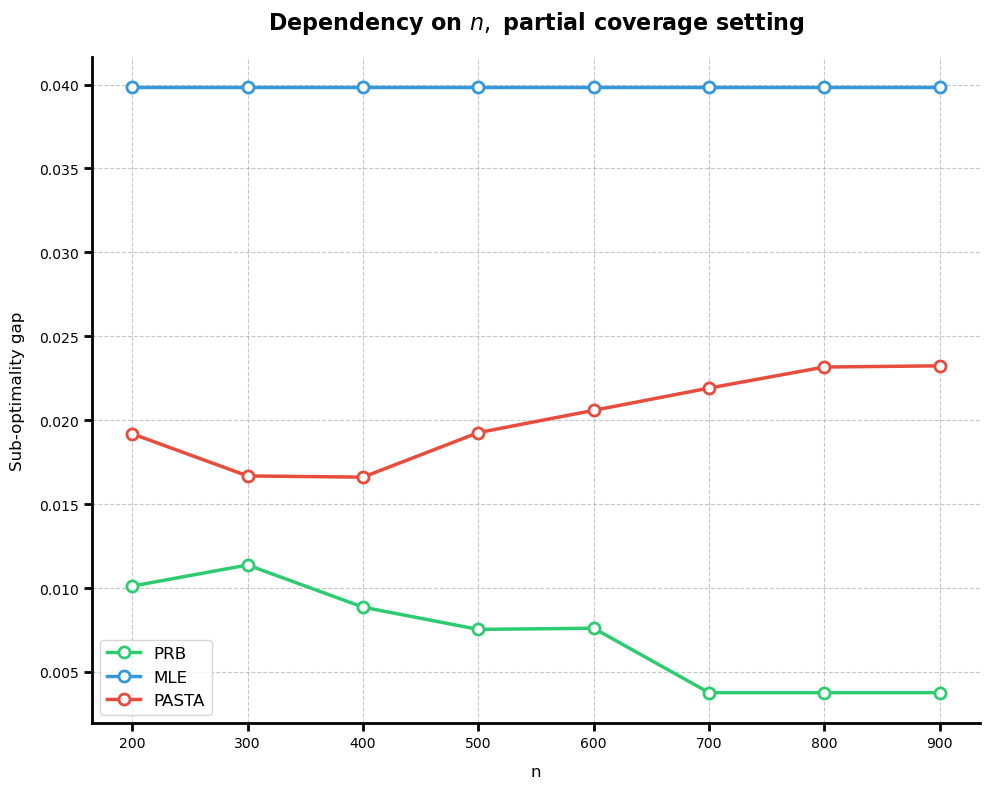}
        \label{fig:1a}
    }
    \hfill  
    \subfloat[$\epsilon = 1/\sqrt{nK}$]{
        \includegraphics[width=0.45\textwidth]
        {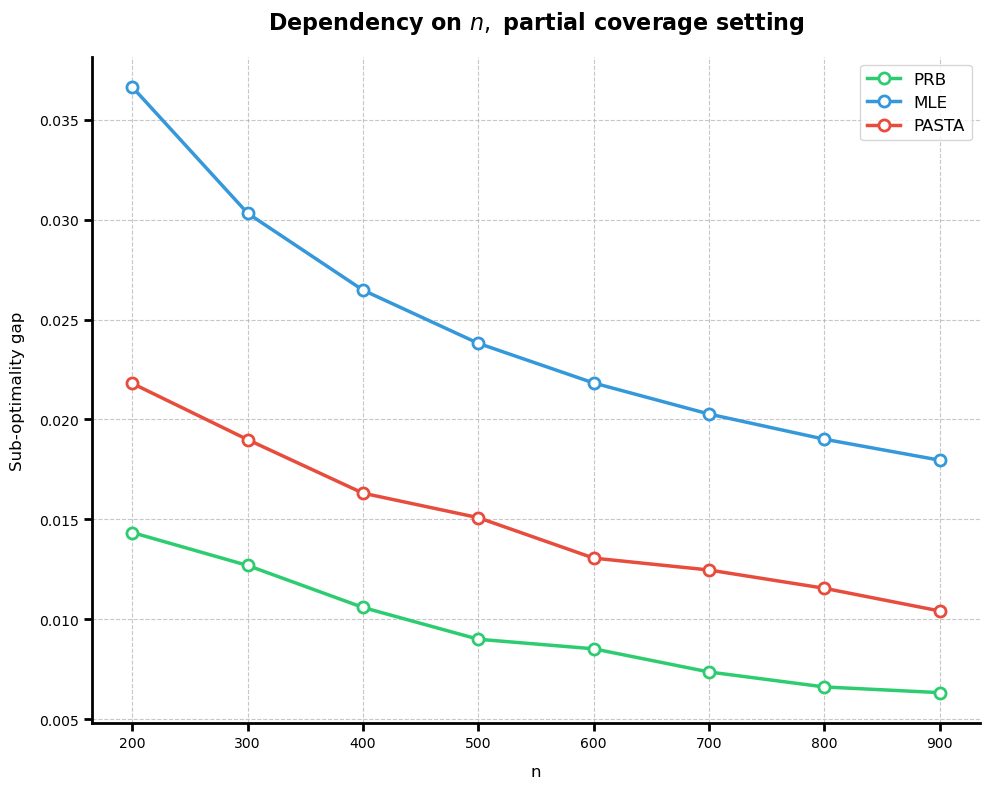}
        \label{fig:1b}
    }
    \caption{Comparison of PRB, PASTA, and MLE in different $\epsilon$ regime}
    \label{fig:dependency-on-n}
\end{figure}

\paragraph{Experimental Setup} In the experiment, we fix $N = 200$ and $K = 5,$ given an effective dataset size $n$, we construct problem parameters and observed datasets as the following: We set $v_0 = 1$ and \begin{align*}
  r_i \equiv 1, \quad \forall i\in[N],  \qquad  \quad  v_i = \begin{cases}
        \frac{1}{K}+ \epsilon, & 1\leq i\leq K,\\
        \frac{1}{K}, & i>K,\\
    \end{cases}
\end{align*}
with $\epsilon=\frac{1}{10\sqrt{Kn}}$. To explain our construction on the observed dataset, for each $1\leq \ell \leq 4K$ and $1\leq j\leq n,$ we denote $$S_{\ell}^{(j)} = \{\ell, m_{1}^{(\ell)},\dots,m_{K-1}^{(\ell)}\}$$
the assortment generated by drawing $\{m_1^{(\ell)},\dots, m_{K-1}^{(\ell)}\}$ uniformly from all possible $K-1$ sized subsets of $[N]\setminus [4K].$ Given $S_{\ell}^{(j)}$, we let  $i_{\ell}^{(j)}$ be the choice sampled from the MNL choice model with $\bm v$ condition on $S_{\ell}^{(j)}$. With this notation, we let the observed dataset $D_{\text{partial}} = \{i_\ell^{(j)},S_{\ell}^{(j)}\}_{j = 1,\ell = 1}^{n,2K}$. It can be seen from our construction that the optimal assortment is given by $S^\star = \{1,\dots,K\}$ and every item in $S^\star$ is covered by the observed assortments exactly for $n$ times. 

\paragraph{Results.} Figure~\ref{fig:dependency-on-n} summarizes the performance of different algorithms in the constant $\epsilon$ regime ($\epsilon = 0.05$, Figure~\ref{fig:1a}) and the small $\epsilon$ regime ($\epsilon = 0.1/\sqrt{n}$, Figure~\ref{fig:1b}). Both results are obtained by running 50 independent simulations and calculating the mean sub-optimality gap. It can be observed that the PRB algorithm outperforms the other two algorithms in both scenarios. In the constant $\epsilon$ regime, the MLE-based algorithm (non-pessimistic) fails to converge as $n$ increases due to partial coverage. The performance of both PASTA is significantly better than that of the non-pessimistic MLE, but not converge as $n$ increases as well.
In the small $\epsilon$ regime, where the choice of $\epsilon$ aligns with our minimax lower bound construction, the results show that PRB continues to outperform the other two algorithms. Notably, the sub-optimality gaps of all algorithms decrease as $n$ increases, which is expected since the gap itself decays with $n$ in this setting.
Finally, it is worth to note that the existing theoretical guarantees for PASTA require observations of the optimal assortment size $n_{S^\star}$ \citep{dong2023pasta}, our simulations show that PASTA still highly outperforms the non-pessimistic baseline even when $n_{S^\star} = 0$. This suggests that the proposed effective number, $\min_{i \in S^\star} n_i$, may be the true complexity measure for PASTA as well. These results open a promising direction for future research on providing sharper offline performance guarantees for PASTA and other MLE-based pessimistic algorithms.

\subsubsection{Dependency on $K$ in uniform and non-uniform reward setting}

\begin{figure}[htbp]
    \centering
    \subfloat[Non-uniform reward setting]{
        \includegraphics[width=0.45\textwidth]{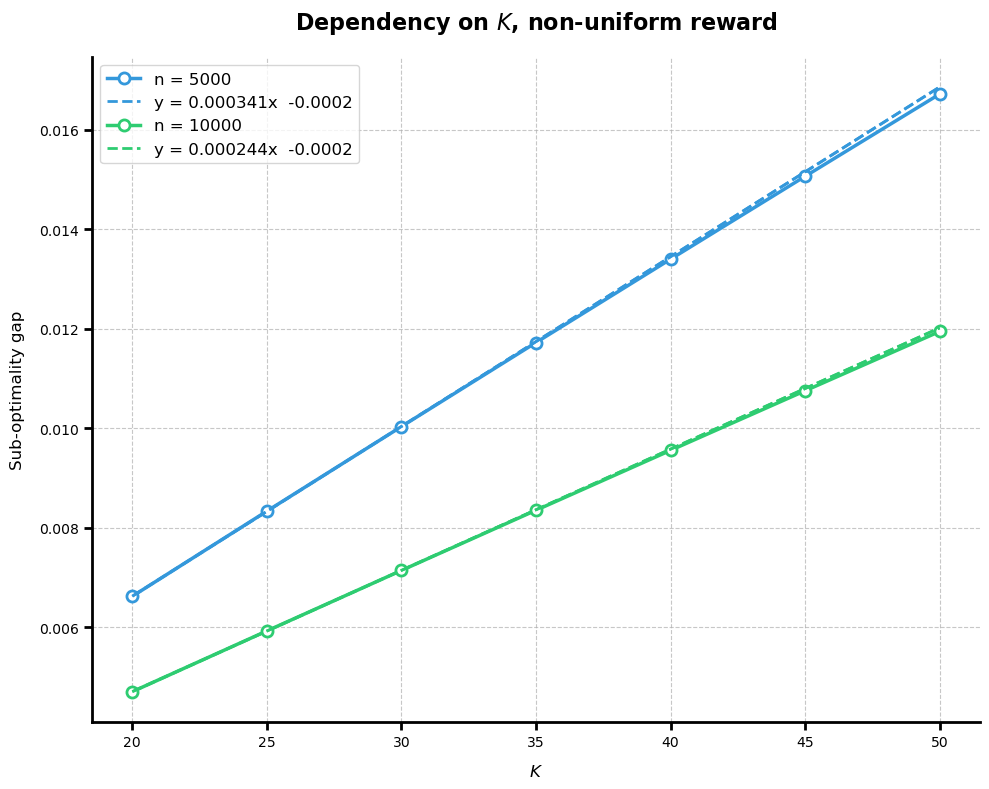}
        \label{fig:2a}
    }
    \hfill  
    \subfloat[Uniform reward setting]{
        \includegraphics[width=0.45\textwidth]{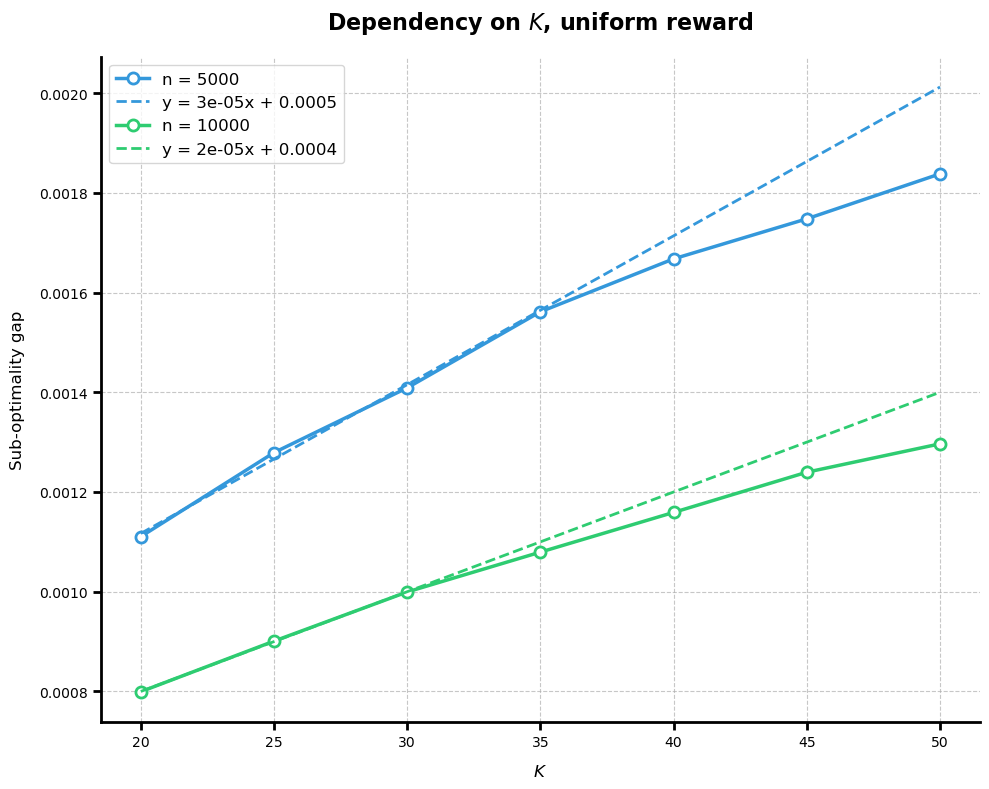}
        \label{fig:2b}
    }
    \caption{Sub-optimality gap of PRB under different $K$}
    \label{fig:dependency-on-K}
\end{figure}

In this section, we fix $N$ and vary $K$ to illustrate the dependency of empirical performance on $K$ in both the uniform reward and non-uniform reward settings. Specifically, we use the construction from our lower bound proof to approximate the worst-case behavior with respect to $K$ in both cases. Due to the high computational cost of MLE and PASTA in the large $K$ and large $n$ regimes (as a sufficiently large $n$ is required in the large $K$ regime to keep the sample complexity guarantee), we limit this experiment to the PRB algorithm to demonstrate the tightness of our lower bound.

\paragraph{Experimental Setup.} In both settings, we fix $N = 200$ and changing $K$ over the set $\{5,10,20,30,40,50\}$ with $n = 5000$ and $n = 10,000$. For each $K$, we use the similar construction of the observed dataset $\{i_{\ell}^{(j)},S_{\ell}^{(j)}\}_{1\leq \ell \leq 4K, 1\leq j\leq n}$ as in section~\ref{sec-experiment-partial-coverage}. For problem parameters $\bm v$ and $\bm r$, we set $v_0 = 1$ and \begin{align*}
  r_i \equiv 1, \quad \forall i\in[N],  \qquad  \quad  v_i = \begin{cases}
        \frac{1}{K}+ \epsilon, & 1\leq i\leq K,\\
        \frac{1}{K}, & i>K,\\
    \end{cases}
\end{align*}
in uniform reward setting with $\epsilon=\frac{1}{10\sqrt{Kn}}$. And $v_0 = 1$
\begin{align*}
  r_i =\begin{cases}
      1, & 1\leq i\leq 4K\\
      0, & i>4K
  \end{cases}, \qquad  \quad  v_i = \begin{cases}
        \frac{1}{K}+ \epsilon', & 1\leq i\leq K,\\
        \frac{1}{K}, & 1<i\leq 4K,\\
        1, & i>4K.
    \end{cases}
\end{align*}
in non-uniform reward setting for $\epsilon' = \frac{1}{10\sqrt{n}}$. 

\paragraph{Results.} Figure~\ref{fig:dependency-on-K} presents the empirical results of PRB in both settings. To determine whether the growth rate with respect to $K$ is linear, we plot the best linear fit for each curve using a dotted line. The linear fit is computed based on the first half of the curve. For linear growth, the curve is expected to align with the linear fit, whereas for sub-linear growth, the curve will fall below it. The results indicate that the sub-optimality gap grows with $K$ at a sub-linear rate in the uniform reward setting, while in the non-uniform reward setting, it grows at a linear rate, consistent with our theoretical minimax rate.

\section{Conclusion}

In this paper, we tackle the challenging problem of offline assortment optimization under the Multinomial Logit (MNL) model. We propose a novel Pessimistic Rank-Breaking (PRB) algorithm under this context, which finds a near-optimal assortment under the optimal item coverage condition --- a much weaker assumption than the optimal assortment coverage required by previous work. Furthermore, by establishing minimal lower bounds, we demonstrate that PRB achieves optimal sample complexity in both uniform and non-uniform reward settings. Through extensive empirical evaluation, we have shown that PRB outperforms existing methods in practical applications, particularly in scenarios with limited observational data. These findings represent a significant advancement in the field of assortment optimization, offering both theoretical insights and practical tools for real-world implementation.

\paragraph{Acknowledgments.}
This work is generously supported by the National Science Foundation under the grants CCF-2106508, CCF-2312204 and CCF-2312205.

\bibliography{reference.bib}

\begin{thebibliography}{51}
\expandafter\ifx\csname natexlab\endcsname\relax\def\natexlab#1{#1}\fi
\expandafter\ifx\csname url\endcsname\relax
  \def\url#1{\texttt{#1}}\fi
\expandafter\ifx\csname urlprefix\endcsname\relax\def\urlprefix{URL }\fi

\bibitem[{Agrawal et~al.(2017)Agrawal, Avadhanula, Goyal and Zeevi}]{agrawal2017thompson}
\textsc{Agrawal, S.}, \textsc{Avadhanula, V.}, \textsc{Goyal, V.} and \textsc{Zeevi, A.} (2017).
\newblock Thompson sampling for the mnl-bandit.
\newblock In \textit{Conference on learning theory}. PMLR.

\bibitem[{Agrawal et~al.(2019)Agrawal, Avadhanula, Goyal and Zeevi}]{agrawal2019mnl}
\textsc{Agrawal, S.}, \textsc{Avadhanula, V.}, \textsc{Goyal, V.} and \textsc{Zeevi, A.} (2019).
\newblock Mnl-bandit: A dynamic learning approach to assortment selection.
\newblock \textit{Operations Research} \textbf{67} 1453--1485.

\bibitem[{Alptekino{\u{g}}lu and Semple(2016)}]{alptekinouglu2016exponomial}
\textsc{Alptekino{\u{g}}lu, A.} and \textsc{Semple, J.~H.} (2016).
\newblock The exponomial choice model: A new alternative for assortment and price optimization.
\newblock \textit{Operations Research} \textbf{64} 79--93.

\bibitem[{Asuncion et~al.(2007)Asuncion, Newman et~al.}]{asuncion2007uci}
\textsc{Asuncion, A.}, \textsc{Newman, D.} \textsc{et~al.} (2007).
\newblock Uci machine learning repository.

\bibitem[{Avadhanula et~al.(2016)Avadhanula, Bhandari, Goyal and Zeevi}]{avadhanula2016tightness}
\textsc{Avadhanula, V.}, \textsc{Bhandari, J.}, \textsc{Goyal, V.} and \textsc{Zeevi, A.} (2016).
\newblock On the tightness of an lp relaxation for rational optimization and its applications.
\newblock \textit{Operations Research Letters} \textbf{44} 612--617.

\bibitem[{Aznag et~al.(2021)Aznag, Goyal and Perivier}]{aznag2021mnl}
\textsc{Aznag, A.}, \textsc{Goyal, V.} and \textsc{Perivier, N.} (2021).
\newblock Mnl-bandit with knapsacks.
\newblock \textit{arXiv preprint arXiv:2106.01135} .

\bibitem[{Berbeglia et~al.(2022)Berbeglia, Garassino and Vulcano}]{berbeglia2022comparative}
\textsc{Berbeglia, G.}, \textsc{Garassino, A.} and \textsc{Vulcano, G.} (2022).
\newblock A comparative empirical study of discrete choice models in retail operations.
\newblock \textit{Management Science} \textbf{68} 4005--4023.

\bibitem[{Blanchet et~al.(2016)Blanchet, Gallego and Goyal}]{blanchet2016markov}
\textsc{Blanchet, J.}, \textsc{Gallego, G.} and \textsc{Goyal, V.} (2016).
\newblock A markov chain approximation to choice modeling.
\newblock \textit{Operations Research} \textbf{64} 886--905.

\bibitem[{Blanchet et~al.(2023)Blanchet, Lu, Zhang and Zhong}]{blanchet2024double}
\textsc{Blanchet, J.}, \textsc{Lu, M.}, \textsc{Zhang, T.} and \textsc{Zhong, H.} (2023).
\newblock Double pessimism is provably efficient for distributionally robust offline reinforcement learning: Generic algorithm and robust partial coverage.
\newblock \textit{Advances in Neural Information Processing Systems} \textbf{36}.

\bibitem[{Caro and Gallien(2007)}]{caro2007dynamic}
\textsc{Caro, F.} and \textsc{Gallien, J.} (2007).
\newblock Dynamic assortment with demand learning for seasonal consumer goods.
\newblock \textit{Management science} \textbf{53} 276--292.

\bibitem[{Chen et~al.(2024)Chen, Liu, Wang and Zhou}]{chen2024re}
\textsc{Chen, X.}, \textsc{Liu, M.}, \textsc{Wang, Y.} and \textsc{Zhou, Y.} (2024).
\newblock A re-solving heuristic for dynamic assortment optimization with knapsack constraints.
\newblock \textit{arXiv preprint arXiv:2407.05564} .

\bibitem[{Chen et~al.(2021{\natexlab{a}})Chen, Shi, Wang and Zhou}]{chen2021dynamic}
\textsc{Chen, X.}, \textsc{Shi, C.}, \textsc{Wang, Y.} and \textsc{Zhou, Y.} (2021{\natexlab{a}}).
\newblock Dynamic assortment planning under nested logit models.
\newblock \textit{Production and Operations Management} \textbf{30} 85--102.

\bibitem[{Chen and Wang(2018)}]{chen2018note}
\textsc{Chen, X.} and \textsc{Wang, Y.} (2018).
\newblock A note on a tight lower bound for capacitated mnl-bandit assortment selection models.
\newblock \textit{Operations Research Letters} \textbf{46} 534--537.

\bibitem[{Chen et~al.(2021{\natexlab{b}})Chen, Wang and Zhou}]{chen2021optimal}
\textsc{Chen, X.}, \textsc{Wang, Y.} and \textsc{Zhou, Y.} (2021{\natexlab{b}}).
\newblock Optimal policy for dynamic assortment planning under multinomial logit models.
\newblock \textit{Mathematics of Operations Research} \textbf{46} 1639--1657.

\bibitem[{Cheung and Simchi-Levi(2017)}]{cheung2017assortment}
\textsc{Cheung, W.~C.} and \textsc{Simchi-Levi, D.} (2017).
\newblock Assortment optimization under unknown multinomial logit choice models.
\newblock \textit{arXiv preprint arXiv:1704.00108} .

\bibitem[{Daganzo(2014)}]{daganzo2014multinomial}
\textsc{Daganzo, C.} (2014).
\newblock \textit{Multinomial probit: the theory and its application to demand forecasting}.
\newblock Elsevier.

\bibitem[{Davis et~al.(2013)Davis, Gallego and Topaloglu}]{davis2013assortment}
\textsc{Davis, J.}, \textsc{Gallego, G.} and \textsc{Topaloglu, H.} (2013).
\newblock Assortment planning under the multinomial logit model with totally unimodular constraint structures.
\newblock \textit{Work in Progress} .

\bibitem[{Davis et~al.(2014)Davis, Gallego and Topaloglu}]{davis2014assortment}
\textsc{Davis, J.~M.}, \textsc{Gallego, G.} and \textsc{Topaloglu, H.} (2014).
\newblock Assortment optimization under variants of the nested logit model.
\newblock \textit{Operations Research} \textbf{62} 250--273.

\bibitem[{Dong et~al.(2023)Dong, Mo, Qi, Shi, Fang and Tarokh}]{dong2023pasta}
\textsc{Dong, J.}, \textsc{Mo, W.}, \textsc{Qi, Z.}, \textsc{Shi, C.}, \textsc{Fang, E.~X.} and \textsc{Tarokh, V.} (2023).
\newblock Pasta: pessimistic assortment optimization.
\newblock In \textit{International Conference on Machine Learning}. PMLR.

\bibitem[{Jin et~al.(2021)Jin, Yang and Wang}]{jin2021pessimism}
\textsc{Jin, Y.}, \textsc{Yang, Z.} and \textsc{Wang, Z.} (2021).
\newblock Is pessimism provably efficient for offline rl?
\newblock In \textit{International Conference on Machine Learning}. PMLR.

\bibitem[{Kumar et~al.(2020)Kumar, Zhou, Tucker and Levine}]{kumar2020conservative}
\textsc{Kumar, A.}, \textsc{Zhou, A.}, \textsc{Tucker, G.} and \textsc{Levine, S.} (2020).
\newblock Conservative q-learning for offline reinforcement learning.
\newblock \textit{Advances in Neural Information Processing Systems} \textbf{33} 1179--1191.

\bibitem[{Le~Cam(2012)}]{le2012asymptotic}
\textsc{Le~Cam, L.} (2012).
\newblock \textit{Asymptotic methods in statistical decision theory}.
\newblock Springer Science \& Business Media.

\bibitem[{Lee and Oh(2024)}]{lee2024nearly}
\textsc{Lee, J.} and \textsc{Oh, M.-h.} (2024).
\newblock Nearly minimax optimal regret for multinomial logistic bandit.
\newblock \textit{arXiv preprint arXiv:2405.09831} .

\bibitem[{Li et~al.(2022)Li, Luo, Huang and Shi}]{li2022onlineassortment}
\textsc{Li, S.}, \textsc{Luo, Q.}, \textsc{Huang, Z.} and \textsc{Shi, C.} (2022).
\newblock Online learning for constrained assortment optimization under markov chain choice model.
\newblock \textit{Available at SSRN 4079753} .

\bibitem[{Liu et~al.(2022)Liu, Lu, Wang, Jordan and Yang}]{liu2022welfare}
\textsc{Liu, Z.}, \textsc{Lu, M.}, \textsc{Wang, Z.}, \textsc{Jordan, M.} and \textsc{Yang, Z.} (2022).
\newblock Welfare maximization in competitive equilibrium: Reinforcement learning for markov exchange economy.
\newblock In \textit{International Conference on Machine Learning}. PMLR.

\bibitem[{Lu et~al.(2023)Lu, Min, Wang and Yang}]{lu2023pessimism}
\textsc{Lu, M.}, \textsc{Min, Y.}, \textsc{Wang, Z.} and \textsc{Yang, Z.} (2023).
\newblock Pessimism in the face of confounders: Provably efficient offline reinforcement learning in partially observable markov decision processes.
\newblock In \textit{The Eleventh International Conference on Learning Representations}.

\bibitem[{Luce(1959)}]{luce1959individual}
\textsc{Luce, R.~D.} (1959).
\newblock \textit{Individual choice behavior}, vol.~4.
\newblock Wiley New York.

\bibitem[{McFadden(1977)}]{mcfadden1977modelling}
\textsc{McFadden, D.} (1977).
\newblock Modelling the choice of residential location .

\bibitem[{McFadden and Train(2000)}]{mcfadden2000mixed}
\textsc{McFadden, D.} and \textsc{Train, K.} (2000).
\newblock Mixed mnl models for discrete response.
\newblock \textit{Journal of applied Econometrics} \textbf{15} 447--470.

\bibitem[{Oh and Iyengar(2021)}]{oh2021multinomial}
\textsc{Oh, M.-h.} and \textsc{Iyengar, G.} (2021).
\newblock Multinomial logit contextual bandits: Provable optimality and practicality.
\newblock In \textit{Proceedings of the AAAI conference on artificial intelligence}, vol.~35.

\bibitem[{Ou et~al.(2018)Ou, Li, Zhu and Jin}]{ou2018multinomial}
\textsc{Ou, M.}, \textsc{Li, N.}, \textsc{Zhu, S.} and \textsc{Jin, R.} (2018).
\newblock Multinomial logit bandit with linear utility functions.
\newblock \textit{arXiv preprint arXiv:1805.02971} .

\bibitem[{Perivier and Goyal(2022)}]{perivier2022dynamic}
\textsc{Perivier, N.} and \textsc{Goyal, V.} (2022).
\newblock Dynamic pricing and assortment under a contextual mnl demand.
\newblock \textit{Advances in Neural Information Processing Systems} \textbf{35} 3461--3474.

\bibitem[{Rakhlin et~al.(2012)Rakhlin, Shamir and Sridharan}]{rakhlin2012making}
\textsc{Rakhlin, A.}, \textsc{Shamir, O.} and \textsc{Sridharan, K.} (2012).
\newblock Making gradient descent optimal for strongly convex stochastic optimization.
\newblock In \textit{Proceedings of the 29th International Coference on International Conference on Machine Learning}.

\bibitem[{Rashidinejad et~al.(2021)Rashidinejad, Zhu, Ma, Jiao and Russell}]{rashidinejad2021bridging}
\textsc{Rashidinejad, P.}, \textsc{Zhu, B.}, \textsc{Ma, C.}, \textsc{Jiao, J.} and \textsc{Russell, S.} (2021).
\newblock Bridging offline reinforcement learning and imitation learning: A tale of pessimism.
\newblock \textit{Advances in Neural Information Processing Systems} \textbf{34} 11702--11716.

\bibitem[{Rashidinejad et~al.(2023)Rashidinejad, Zhu, Yang, Russell and Jiao}]{rashidinejad2023optimal}
\textsc{Rashidinejad, P.}, \textsc{Zhu, H.}, \textsc{Yang, K.}, \textsc{Russell, S.} and \textsc{Jiao, J.} (2023).
\newblock Optimal conservative offline {RL} with general function approximation via augmented lagrangian.
\newblock In \textit{The Eleventh International Conference on Learning Representations}.

\bibitem[{Rusmevichientong et~al.(2010)Rusmevichientong, Shen and Shmoys}]{rusmevichientong2010dynamic}
\textsc{Rusmevichientong, P.}, \textsc{Shen, Z.-J.~M.} and \textsc{Shmoys, D.~B.} (2010).
\newblock Dynamic assortment optimization with a multinomial logit choice model and capacity constraint.
\newblock \textit{Operations research} \textbf{58} 1666--1680.

\bibitem[{Saha and Gaillard(2024)}]{saha2024stop}
\textsc{Saha, A.} and \textsc{Gaillard, P.} (2024).
\newblock Stop relying on no-choice and do not repeat the moves: Optimal, efficient and practical algorithms for assortment optimization.
\newblock \textit{arXiv preprint arXiv:2402.18917} .

\bibitem[{Saha and Gopalan(2019)}]{saha2019active}
\textsc{Saha, A.} and \textsc{Gopalan, A.} (2019).
\newblock Active ranking with subset-wise preferences.
\newblock In \textit{The 22nd International Conference on Artificial Intelligence and Statistics}. PMLR.

\bibitem[{Saur{\'e} and Zeevi(2013)}]{saure2013optimal}
\textsc{Saur{\'e}, D.} and \textsc{Zeevi, A.} (2013).
\newblock Optimal dynamic assortment planning with demand learning.
\newblock \textit{Manufacturing \& Service Operations Management} \textbf{15} 387--404.

\bibitem[{Shi et~al.(2022)Shi, Li, Wei, Chen and Chi}]{shi2022pessimistic}
\textsc{Shi, L.}, \textsc{Li, G.}, \textsc{Wei, Y.}, \textsc{Chen, Y.} and \textsc{Chi, Y.} (2022).
\newblock Pessimistic q-learning for offline reinforcement learning: Towards optimal sample complexity.
\newblock In \textit{International conference on machine learning}. PMLR.

\bibitem[{Talluri and Van~Ryzin(2004)}]{talluri2004revenue}
\textsc{Talluri, K.} and \textsc{Van~Ryzin, G.} (2004).
\newblock Revenue management under a general discrete choice model of consumer behavior.
\newblock \textit{Management Science} \textbf{50} 15--33.

\bibitem[{Train(2009)}]{train2009discrete}
\textsc{Train, K.~E.} (2009).
\newblock \textit{Discrete choice methods with simulation}.
\newblock Cambridge university press.

\bibitem[{Tsybakov(2008)}]{tsybakov2009nonparametric}
\textsc{Tsybakov, A.~B.} (2008).
\newblock \textit{Introduction to Nonparametric Estimation}.
\newblock 1st ed. Springer Publishing Company, Incorporated.

\bibitem[{Uehara and Sun(2021)}]{uehara2021pessimistic}
\textsc{Uehara, M.} and \textsc{Sun, W.} (2021).
\newblock Pessimistic model-based offline reinforcement learning under partial coverage.
\newblock In \textit{International Conference on Learning Representations}.

\bibitem[{Xie et~al.(2021)Xie, Cheng, Jiang, Mineiro and Agarwal}]{xie2021bellman}
\textsc{Xie, T.}, \textsc{Cheng, C.-A.}, \textsc{Jiang, N.}, \textsc{Mineiro, P.} and \textsc{Agarwal, A.} (2021).
\newblock Bellman-consistent pessimism for offline reinforcement learning.
\newblock \textit{Advances in neural information processing systems} \textbf{34} 6683--6694.

\bibitem[{Xiong et~al.(2023)Xiong, Zhong, Shi, Shen, Wang and Zhang}]{xiong2023nearly}
\textsc{Xiong, W.}, \textsc{Zhong, H.}, \textsc{Shi, C.}, \textsc{Shen, C.}, \textsc{Wang, L.} and \textsc{Zhang, T.} (2023).
\newblock Nearly minimax optimal offline reinforcement learning with linear function approximation: Single-agent mdp and markov game.
\newblock In \textit{International Conference on Learning Representations (ICLR)}.

\bibitem[{Yu(1997)}]{yu1997assouad}
\textsc{Yu, B.} (1997).
\newblock Assouad, fano, and le cam.
\newblock In \textit{Festschrift for Lucien Le Cam: research papers in probability and statistics}. Springer, 423--435.

\bibitem[{Yu et~al.(2020)Yu, Thomas, Yu, Ermon, Zou, Levine, Finn and Ma}]{yu2020mopo}
\textsc{Yu, T.}, \textsc{Thomas, G.}, \textsc{Yu, L.}, \textsc{Ermon, S.}, \textsc{Zou, J.~Y.}, \textsc{Levine, S.}, \textsc{Finn, C.} and \textsc{Ma, T.} (2020).
\newblock Mopo: Model-based offline policy optimization.
\newblock \textit{Advances in Neural Information Processing Systems} \textbf{33} 14129--14142.

\bibitem[{Zhan et~al.(2022)Zhan, Huang, Huang, Jiang and Lee}]{zhan2022offline}
\textsc{Zhan, W.}, \textsc{Huang, B.}, \textsc{Huang, A.}, \textsc{Jiang, N.} and \textsc{Lee, J.} (2022).
\newblock Offline reinforcement learning with realizability and single-policy concentrability.
\newblock In \textit{Conference on Learning Theory}. PMLR.

\bibitem[{Zhang and Sugiyama(2024)}]{zhang2024online}
\textsc{Zhang, Y.-J.} and \textsc{Sugiyama, M.} (2024).
\newblock Online (multinomial) logistic bandit: Improved regret and constant computation cost.
\newblock \textit{Advances in Neural Information Processing Systems} \textbf{36}.

\bibitem[{Zhong et~al.(2022)Zhong, Xiong, Tan, Wang, Zhang, Wang and Yang}]{zhong2022pessimistic}
\textsc{Zhong, H.}, \textsc{Xiong, W.}, \textsc{Tan, J.}, \textsc{Wang, L.}, \textsc{Zhang, T.}, \textsc{Wang, Z.} and \textsc{Yang, Z.} (2022).
\newblock Pessimistic minimax value iteration: Provably efficient equilibrium learning from offline datasets.
\newblock In \textit{International Conference on Machine Learning}. PMLR.

\end{thebibliography}

\bibliographystyle{ims}

\newpage 

\appendix

\section{Proof of Upper Bound Results}

We first present a general result and derive both Theorem~\ref{thm-upper-bound-non-uniform} and Theorem~\ref{thm-upper-bound-uniform} from this result.

\begin{theorem}[Sub-Optimality Gap for Algorithm~\ref{alg-rank-breaking}]\label{thm-upper-bound} Given \begin{align*}
\sum_{k =1}^n \frac{\bm{1}\{i\in S_k\}}{1+\sum_{j\in S_k }v_j} \geq 256\log(N/\delta),\quad \forall i\in S^\star,
\end{align*}
we have then with probability at least $1-2\delta$, it holds that
\begin{align*}
        &\mathrm{SubOpt}( \widehat{S}; \bm{v})  \\
        &\qquad \le \sum_{i\in S^\star} \frac{r_i}{1+\sum_{i\in S^\star} v_i} \cdot \left(\sqrt{\frac{4v_i(1+v_i)\log (N/\delta) }{\sum_{k = 1}^n \bm{1}\{i\in S_k\}(1+\sum_{j\in S_k}v_j)^{-1}}}  + \frac{48\log(N/\delta) (1+v_i)}{\sum_{k = 1}^n \bm{1}\{i\in S_k\}(1+\sum_{j\in S_k}v_j)^{-1}} \right).
\end{align*}
\end{theorem}

\begin{proof}[Proof of Theorem~\ref{thm-upper-bound}]
    Please see Appendix~\ref{subsec: proof upper bound} for a detailed proof.
\end{proof}

\subsection{Proof of Theorem~\ref{thm-upper-bound-non-uniform}}\label{subsec: proof upper bound non uniform}

\proof[Proof of Theorem~\ref{thm-upper-bound-non-uniform}]
For the first term in Theorem~\ref{thm-upper-bound}, we have \begin{align*}
&\sum_{i\in S^\star} \frac{r_i}{1+\sum_{i\in S^\star} v_i} \sqrt{\frac{4v_i(1+V)\log (N/\delta) }{\sum_{k = 1}^n \bm{1}\{i\in S_k\}(1+\sum_{j\in S_k}v_j)^{-1}}} \\
& \qquad\leq \sqrt{4(1+V)(1+KV)\log(N/\delta)} \cdot \sum_{i\in S^\star} \frac{\sqrt{v_i/n_i}}{1+\sum_{j\in S^\star} v_j}\\
&\qquad \leq \sqrt{\frac{4(1+V)(1+KV)\log(N/\delta)}{1+\sum_{j\in S^\star} v_j}}  \cdot\sqrt{\sum_{i\in S^\star}\frac{1}{n_i}}\\
&\qquad \leq (1+V)K\cdot\sqrt{\frac{4\log (N/\delta)}{\min_{i\in S^\star}n_i}},
\end{align*}
where the first inequality uses the size constraint and that $v_i\leq V$ and $r_i\leq 1$, the second inequality applies Cauchy-Schwartz inequality, and the last inequality also uses the size constraint.
For the second term, we have \begin{align*}
\sum_{i\in S^\star} \frac{r_i}{1+\sum_{i\in S^\star} v_i}\cdot  \frac{48(1+v_i)\log (N/\delta) }{\sum_{k = 1}^n \bm{1}\{i\in S_k\}(1+\sum_{j\in S_k}v_j)^{-1}} & \leq  \frac{48K^2(1+V)^2\log(N/\delta)}{\min_{i\in S^\star} n_i}.
\end{align*}
This finishes the proof.
\endproof

\subsection{Proof of Theorem~\ref{thm-upper-bound-uniform}}\label{subsec: proof upper bound uniform}

\proof[Proof of Theorem~\ref{thm-upper-bound-uniform}]
In the uniform reward setting, we always have \begin{align}
    \sum_{i\in S^\star} v_i \geq \sum_{i\in S} v_i,\quad \forall S \text{ with }\lvert S \rvert \leq K.\label{eq: uniform condition}
\end{align}
Consequently, the first term in Theorem~\ref{thm-upper-bound} can be simplified to
\allowdisplaybreaks
\begin{align*}
  &\sum_{i\in S^\star} \frac{r_j}{1+\sum_{i\in S^\star} v_i}\cdot \sqrt{\frac{4v_i(1+v_i)\log (N/\delta) }{\sum_{k = 1}^n \bm{1}\{i\in S_k\}(1+\sum_{j\in S_k}v_j)^{-1}}}\\
  &\qquad \leq   \sum_{i\in S^\star} \frac{r_j}{1+\sum_{i\in S^\star} v_i} \cdot\sqrt{\frac{4v_i(1+v_i)\log (N/\delta) }{\sum_{k = 1}^n \bm{1}\{i\in S_k\}(1+\sum_{j\in S^\star}v_j)^{-1}}} \\
  &\qquad \leq \sum_{i\in S^\star} \sqrt{\frac{4v_i(1+v_i)\log (N/\delta) }{n_i (1+\sum_{i\in S^\star} v_i)}} \\
  &\qquad \leq 
2\sqrt{\frac{(1+V)\log(N/\delta)}{\min_{i\in S^\star} n_i}}  \cdot \frac{\sum_{i\in S^\star}\sqrt{v_i}}{\sqrt{1+\sum_{i\in S^\star} v_i} } \\
&\qquad \leq 2\sqrt{\frac{K(1+V)\log(N/\delta)}{\min_{i\in S^\star} n_i}},
\end{align*}
where the first inequality uses \eqref{eq: uniform condition}, the second inequality uses $r_i\leq 1$, the third inequality uses  $v_i\leq V$, and the last inequality uses Cauchy-Schwartz inequality.
Similarly, for the second term in Theorem~\ref{thm-upper-bound}, we have 
\begin{align*}
    \sum_{i\in S^\star} \frac{r_i}{1+\sum_{i\in S^\star} v_i}\frac{48\log(N/\delta) (1+v_i)}{\sum_{k = 1}^n \bm{1}\{i\in S_k\}(1+\sum_{j\in S_k}v_j)^{-1}} &\leq \frac{48K(1+V)\log(N/\delta)}{\min_{i\in S^\star} n_i},
\end{align*}
as desired, finishing the proof.
\endproof

\subsection{Proof of Lemma~\ref{lem-SubOptGap-via-pessmistic}}\label{subsec: proof pessimistic}

\begin{proof}[Proof of Lemma~\ref{lem-SubOptGap-via-pessmistic}]
    We first prove a general monotone result of the revenue function $R$ under $\bm v,$ similar to those in Lemma~A.3 of \citet{agrawal2019mnl}:
    \begin{lemma}\label{lem-monotone}
    Consider any two positive vectors $\bm v, \bm v'$ with $\bm v \leq \bm v'$ element-wise, and let $S_{\bm v}$ denote the optimal assortment under $\bm v,$ we have then $R(S_{\bm v};\bm v')\geq R(S_{\bm v};\bm v).$    
    \end{lemma}
    \begin{proof}[Proof of Lemma~\ref{lem-monotone}]
    First noticing that by the optimality of $S_{\bm v}$ under $\bm v,$ we have $ R(S_{\bm v};\bm v) \leq r_j, \forall j \in S_{\bm v}.$ Now by the identity $$ \frac{a+(a/b)\Delta}{b+\Delta} = \frac{a}{b}, \forall b>0,a,\Delta\geq 0,$$ it holds that \begin{align*}
    R(S_{\bm v};\bm v') = \frac{\sum_{i\in S_{\bm v}}r_iv_i'}{1+\sum_{i\in S_{\bm v}}v_i' } \geq \frac{\sum_{i\in S_{\bm v}} r_iv_i +R(S_{\bm v};\bm v)\sum_{i\in S_{\bm v}}(v_i' - v_i)}{1+\sum_{i\in S_{\bm v}}v_i + \sum_{i\in S_{\bm v}} (v_i'-v_i) }= R(S_{\bm v};\bm v).
\end{align*}
    \end{proof}
With Lemma~\ref{lem-monotone}, we have then $R(\underline{S};\underline{\bm v}) \leq R(\underline{S};\bm v)$ and \begin{align*}
    R(S^\star;\bm v) -   R(\underline{S};{\bm v})&\leq  R(S^\star;\bm v) -   R(\underline{S};\underline{\bm v})\leq R(S^\star;\bm v) - R(S^\star;\underline{\bm v})\\
    &= \sum_{j\in S^\star}(\frac{r_jv_j}{1+\sum_{j\in S^\star} v_j} - \frac{r_j\underline{v}_j}{1+\sum_{j\in S^\star} \underline{v}_j})\\
    &\leq \sum_{j\in S^\star}(\frac{r_jv_j}{1+\sum_{j\in S^\star} v_j} - \frac{r_j\underline{v}_j}{1+\sum_{j\in S^\star} {v}_j})= \sum_{j\in S^\star}\frac{r_j(v_j - \underline{v}_j)}{1+\sum_{j\in S^\star} v_j},
\end{align*}
where in the last inequality we have used $
    \frac{1}{1+\sum_{j\in S^\star}\underline{v}_j} \geq \frac{1}{1+\sum_{j\in S^\star} v_j}.$
    
\end{proof}

\subsection{Proof of Theorem~\ref{thm-upper-bound}}\label{subsec: proof upper bound} 

\begin{proof}[Proof of Theorem~\ref{thm-upper-bound}]
It has been shown in \citet{saha2024stop} that the PRB algorithm satisfies the following finite sample error bound holds for $\{\vLCB_j\}_{j\in [N]}$ calculated in the PRB algorithm:
\begin{proposition}\label{prop-vLCB}
   With probability at least $1-3 Ne^{-\delta}$, it holds simultaneously for all $i \in[N]$ that $ v^{\text{LCB}}_i \leq v_{i}$ and one of the following two inequalities is satisfied
$$
n_{i 0 }<69 \delta\left(1+v_i\right)
$$
or
$$
v_{i}^{\text{LCB}} \geq v_i - 4\left(1+v_i\right) \sqrt{\frac{2 v_i \delta}{\tau_{i 0}}}+\frac{22 \delta\left(1+v_i\right)^2}{\tau_{i 0}}.
$$
\end{proposition}

On the other hand, W.L.O.G. suppose that $i\in S_k$ for $k\in [n_i]$, then  \begin{align*}
    \tau_{i0} = \sum_{k = 1}^{n_i} Z_k 
\end{align*}
with $Z_k = \bm{1}\{i_k\in \{0,i\} \}\sim \text{Bernoulli}(\frac{1+v_i }{1+\sum_{j\in S_k} v_j})$ condition on $Z_1,\dots,Z_{k-1}$. As a result, \begin{align*}
    \tau_{i0} - (1+v_i)\sum_{k = 1}^{n_i}\frac{1}{1+\sum_{j\in S_k} v_j} = \sum_{k = 1}^{n_i}\underbrace{\big(\bm{1}\{i_k\in \{0,i\}\} - \frac{1+v_i}{1+\sum_{j\in S_k} v_j} \big)}_{:= D_i}
\end{align*}
is a summation of martingale difference sequences $\{D_k\}$ with $$\text{Var}(D_k\lvert D_1,\dots,D_{k-1})= \frac{(1+v_i)\sum_{j\in S_k,j\neq i} v_j}{(1+\sum_{j\in S_k} v_j)^2} \leq \frac{1+v_i}{1+\sum_{j\in S_k} v_j}.$$
Now applying Freedman's inequality(see e.g. Lemma~3 in \cite{rakhlin2012making}) implies with probability at least $1-\delta'\log n,$
\begin{align*}
    \tau_{i0} - \sum_{k = 1}^{n_i}\frac{(1+v_i)}{1+\sum_{j\in S_k} v_j}\geq -2 \max\left\{2\sqrt{\sum_{k = 1}^{n_i}\frac{(1+v_i)}{1+\sum_{j\in S_k} v_j}} ,\sqrt{\log(1/\delta')}\right\} \cdot \sqrt{\log(1/\delta')}.
\end{align*}
Selecting $\delta' = \delta/(N\log N)$ and noticing that $\log((N\log N )/\delta)\leq 2 \log(N/\delta)$, we have with probability at least $1-\delta/N,$ it holds that \begin{align}\label{eq-ni0-chernoff-bound}
    \begin{aligned}
    &\sum_{k = 1}^{n_i} \frac{1}{1+\sum_{j\in S_k} v_j}\geq 256 \log(N/\delta)\\
    &\qquad \implies 2 \max\left\{2\sqrt{\sum_{k = 1}^{n_i}\frac{(1+v_i)}{1+\sum_{j\in S_k} v_j}} ,\sqrt{\log(1/\delta')}\right\} = 4\sqrt{\sum_{k = 1}^{n_i}\frac{(1+v_i)}{1+\sum_{j\in S_k} v_j}}\\
    &\qquad\implies 4\sqrt{\sum_{k = 1}^{n_i}\frac{(1+v_i)}{1+\sum_{j\in S_k} v_j} \cdot 2\log (N/\delta)}\leq \frac{1}{2} \sum_{k = 1}^{n_i} \frac{1+v_i}{1+\sum_{j\in S_k} v_j} \\
    &\qquad\implies \tau_{i0} \geq \frac{1}{2}\sum_{k = 1}^{n_i} \frac{1+v_i}{1+\sum_{j\in S_k} v_j} \geq 69 \log(N/\delta)
    \end{aligned}
\end{align}
Now plugging Proposition~2 and \eqref{eq-ni0-chernoff-bound} into the Lemma~\ref{lem-SubOptGap-via-pessmistic}, we have that given the condition \begin{align*}
\sum_{k =1}^n \frac{\bm{1}\{i\in S_k\}}{1+\sum_{j\in S_k }v_j} \geq 256\log(N/\delta),\quad \forall i\in S^\star,
\end{align*}
it holds that with probability at least $1-2\delta$, 
\begin{align*}
        &\mathrm{SubOpt}( \widehat{S}; \bm{v}) \\
        &\qquad \le \frac{1}{1+\sum_{i\in S^\star} v_i}  \sum_{i\in S^\star} r_i\cdot \left(\sqrt{\frac{4v_i(1+v_i)\log (N/\delta) }{\sum_{k = 1}^n \bm{1}\{i\in S_k\}(1+\sum_{j\in S_k}v_j)^{-1}}}  + \frac{48\log(N/\delta) (1+v_i)}{\sum_{k = 1}^n \bm{1}\{i\in S_k\}(1+\sum_{j\in S_k}v_j)^{-1}} \right).
\end{align*}
This completes the proof of Theorem~\ref{thm-upper-bound}.
\end{proof}

\section{Proof of Auxiliary Results in Section~\ref{sec-lower-bound-proof}}\label{sec: proof lemma lower bound}

\subsection{Proof of Lemma~\ref{lem-packing-number}}

\begin{proof}[Proof of Lemma~\ref{lem-packing-number}]
We assume without loss of generality that $K/8$ is an integer, otherwise we can replace $K/4,K/8$ by $\lfloor K/4\rfloor,\lfloor K/8\rfloor$ all the analysis still hold. Now for every $K$-sized subset $S$ of $[4K]$, define its $K/4$-neighborhood under $\Delta$ as \begin{align*}
    \delta(S,K/4):= \{S'\subset [4K]: \lvert S' \rvert = K, \Delta(S,S') \leq K/4 \}.
\end{align*}

We have then denote $\mathcal{F}$ as the maximal $K/2$-packing set consists of $K$-sized subsets of $[4K]$, we have \begin{align*}
    \lvert \delta(S,K/4)\rvert \cdot \lvert \mathcal{F} \rvert \geq \binom{4K}{K}.
\end{align*}
Since otherwise there will exists some $K$-sized set does not lies in any $K/4$ neighborhood of $S\in \mathcal{F},$ a contradiction to $\mathcal{F}$ is maximal. 

Now it suffice to provide lower bound for $\binom{4K}{K}/\lvert \delta(S,K/4)\rvert:$

\paragraph{Upper bound for $\delta(S,K/2)$.} For any  $S,S'\subset [4K]$ with $\lvert S \rvert = \lvert S'\rvert = K$, $\Delta(S,S')$ must be even and for $m\leq K/4$ we have \begin{align*}
    \lvert S\cap S'\rvert  = K-m \iff \lvert S\cup S'\rvert = K+m  \iff \Delta(S,S') = 2m,  
\end{align*}
as a result, we have \begin{align*}
    \lvert \delta(S,K/4) \rvert &\leq 1+\sum_{m=1}^{ K/8 }\lvert \{S':\lvert S'\rvert = K, \lvert S\cap S'\rvert = m\}\rvert \\
    &=1+ \sum_{m=1}^{ K/8 } \binom{K}{m}\binom{3K}{K-m} \\
    &\leq \frac{K}{2}\cdot \max_{1\leq m\leq  K/8 } \underbrace{\binom{K}{m}\binom{3K}{K-m}}_{:= f(m)}.
\end{align*}

Now since for every $1\leq m\leq  K/8,$ we have \begin{align*}
    \frac{f(m+1)}{f(m)} = \frac{\binom{K}{m+1}\binom{3K}{K-(m+1)} }{\binom{K}{m}\binom{3K}{K-m}} = \frac{(K-m)^2}{(m+1)(2K+m+1)},
\end{align*}
thus for every $m< K/4,$ we have \begin{align*}
    \frac{(K-m)^2}{(m+1)(2K+m+1)} > \frac{{9}K^2/{16}}{\frac{K}{4}\cdot \frac{9}{4}K   } =1,
\end{align*}
i.e. the maximal of $f(\cdot)$ on $\{1,\dots, K/8 \}$ is attained at $m = K/8.$

As a result, we get \begin{align*}
    \frac{K}{2}\cdot \lvert \mathcal{F}\rvert &\geq \frac{\binom{4K}{K}}{f( \frac{K}{8})} = \frac{\binom{4K}{K}}{\binom{K}{ \frac{K}{8}}\binom{3K}{\frac{7K}{8}}}= \frac{(4K)!}{K!(3K)!}\bigg/\bigg( \frac{K!(3K)!}{(\frac{7K}{8})!( \frac{K}{8} )! \cdot (\frac{7K}{8})! (\frac{17K}{8} )!} \bigg).
\end{align*}
Now applying Stirling's approximation, where there exists some absolute constant $c>0$ so that \begin{align}\label{chx-eq: strling}
   e^{-c/n} \leq  \frac{n!}{\sqrt{2\pi n} (n/e)^n}  \leq e^{c/n},\quad \forall n\geq 1,
\end{align}
we have there exists some absolute constants $c_1,c_2,c'>0$ so that
\begin{align*}
    &\frac{(4K)!}{K!(3K)!}\bigg/\bigg( \frac{K!(3K)!}{(\frac{7K}{8})!( \frac{K}{8} )! \cdot (\frac{7K}{8})! (\frac{17K}{8} )!} \bigg) \\
    &= (4K)! (\frac{K}{8})! (\frac{17K}{8})! \cdot\bigg(\frac{(7K/8)!}{K!(3K)!} \bigg)^2\\
    &\geq_{\text{(i)}} c_1 e^{-c_2/K}\sqrt{K} \frac{(4K)^{4K} (K/8)^{K/8} (17K/8)^{17K/8} (7K/8)^{7K/4}}{(3K)^{6K} K^{2K}}\\
    &\geq c'\sqrt{K} \bigg(\underbrace{\frac{4^4 (1/8)^{1/8} (17/8)^{17/8} (7/8)^{7/4} }{3^6 }}_{\geq \frac{26}{25}}\bigg)^K.
\end{align*}
Where (i) is by~\eqref{chx-eq: strling}. 
Thus we can get\begin{align*}
    \lvert \mathcal{ F} \rvert \geq \frac{2c'}{\sqrt{K}} \cdot (\frac{26}{25})^K, 
\end{align*}
as a result, there exists some absolute number $C_0 >0$ so that when $K >C_0,$ we have then
\begin{align*}
    \frac{2C'}{\sqrt{K}} (\frac{26}{25})^{K/2} > 1 \implies \lvert \mathcal{ F} \rvert \geq (\frac{26}{25})^{K/2} \implies \log \lvert  \mathcal{F}\rvert &\geq \underbrace{\frac{\log(\frac{26}{25})}{2}}_{:= C_1} K,
\end{align*}
as desired.
\end{proof}

\subsection{Proof of Lemma~\ref{lem-lb-via-distance}}

\begin{proof}[Proof of Lemma~\ref{lem-lb-via-distance}]

Firstly, notice that for any $S \in \mathcal{S}_K$ that contains some $i\in \Nsubopt,$ we have replacing 
such $i$ to any other $\widetilde{i} \in \Nc(\bm r, \bm v)$ will not change $\Delta(\Nopt(\bm r, \bm v),S)$ and will not decrease the revenue of $S$, thus it sufficient to consider those $S\subset [4K].$ For any such $S,$ by $r_i \equiv 1$, we have denoting $n_{S,S'} =\lvert S\cap S' \rvert, $ then
\begin{align*}
R(\Nopt(\bm r, \bm v);\bm{r}, \bm{v}) - R(S;\bm{r}, \bm{v}) &= \frac{1+K\epsilon}{2 +K\epsilon} - \frac{1 + \lvert S\cap S' \rvert \epsilon}{2 + \lvert S \cap S'\rvert \epsilon}\\
&\geq \frac{(K - n_{S,S'})\epsilon}{\big(2+K\epsilon)^2}\geq \frac{\epsilon}{18}\Delta(\Nopt(\bm r, \bm v), S),
\end{align*}
where in the last inequality we have used $K\epsilon\leq 1$ and $K-n_{S,S'} \geq \frac{1}{2}\Delta(S,S' )$.
\end{proof}

\subsection{Proof of Lemma~\ref{lem-KL-new-upper-bound-non-uniform} }
\begin{proof}[Proof of Lemma~\ref{lem-KL-new-upper-bound-non-uniform}]
Denoting $\bm q, \bm p$ the mass of distribution induced by $\bm {v}$ and $\bm {v'}$ under $S^{(\ell)}_1$ respectively. 
Then as shown in \citet{chen2018note,lee2024nearly}, the KL divergence is upper bounded as \begin{align*}
    \sum_{j\in (S^{(\ell)}_1)_+} p_j \log \frac{p_j}{q_j} \leq \sum_{j\in (S^{(\ell)}_1)_+}\frac{p_j(p_j - q_j)}{q_j}= \sum_{j\in (S^{(\ell)}_1)_+}\frac{(p_j - q_j)^2}{q_j} 
\end{align*} 
where we have used $\log (1+x) \leq x$ for $x > -1$ in the first inequality and $\sum_{j\in (S^{(\ell)}_1)_+} (p_j - q_j) = 0$ in 
the second one.

Now if $\ell \notin (\Nopt(\bm r, \bm v) \cup \Nopt(\bm r', \bm v'))\setminus(\Nopt(\bm r, \bm v)\cap \Nopt(\bm r', \bm v'))$ we have then $p_i =q_i,\quad \forall i \in (S^{(\ell)}_1)_+,$ that leads to the second inequality.

To prove the first inequality, we divide the upper bound of the last summation into two cases:

\noindent i) For $j \neq \ell,$ we have by $j\in \Nsubopt$ and $v_j = v_j' =1$, \begin{align*}
    \lvert p_j - q_j\rvert &=  \Big| \frac{1}{K +1/K+ \epsilon} - \frac{1}{K+1/K} \Big|= \frac{\epsilon}{(K+1/K)(K+1/K + \epsilon)} \leq \frac{\epsilon}{K^2},
\end{align*}
thus then 
\begin{align*}
    \frac{(p_j - q_j)^2}{q_j} & \leq \frac{\epsilon^2(K+1/K) }{K^4}  \leq \frac{2\epsilon^2}{K^3} .
\end{align*}

\noindent ii) For $j = \ell,$ we have \begin{align*}
    \lvert p_j - q_j \rvert = \Big| \frac{1/K+\epsilon}{K+1/K + \epsilon} - \frac{1/K}{K+1/K}\Big| =  \frac{(K+1/K)(1/K+\epsilon) - (K+1/K+\epsilon)1/K}{(K+1/K)(K+1/K+\epsilon)} \leq \frac{\epsilon }{K},
\end{align*}
which then implies \begin{align*}
\frac{(p_i - q_i)^2}{q_i} & \leq  \frac{\epsilon^2}{K^2}\cdot (2K)^2 \leq 4\epsilon^2.
\end{align*}
Combining above two cases, we get then \begin{align*}
\sum_{j\in (S^{(\ell)}_1)_+} p_j \log \frac{p_j}{q_j} & \leq 5\epsilon^2,
\end{align*}
as desired.
\end{proof}

\section{Proof of Lower Bound for Uniform Reward (Theorem~\ref{thm-lower-bound-uniform})}\label{appendix-proof-lower-bound-uniform}

\subsection{Modifications on the hard-instance construction}

Here we summarize the change we made for uniform-reward setting on the construction in section~\ref{sec-construction}:
\begin{enumerate}
    \item For items in $\Nopt$, we modify the attraction value to $\frac{1}{K} +\epsilon' $ with $\epsilon' < \frac{1}{4K}$ to be specified later.
    \item For items in $\Nsubopt$, we modify the attraction value to $\frac{1}{K} $ and revenue to $1.$
\end{enumerate}

It can be verified directly that after such modification, Lemma~\ref{lem-lb-via-distance} still holds, with $\epsilon$ replaced by $\epsilon'.$ Now with Lemma~\ref{lem-lb-via-distance}, we can reduce the proof of Theorem~\ref{thm-lower-bound-uniform} to showing that \begin{equation}\label{eq-reduce-to-distance-lb-uniform}
        \min_{\pi}\max_{(\bm r, \bm v) \in \mathcal{V}} \E[ \Delta (\Nopt(\bm r, \bm v),\pi (D))] \gtrsim  (\epsilon')^{-1}\sqrt{\frac{K}{\nmin}}.
\end{equation}

Similar as in non-uniform reward setting, by applying Fano's lemma, we arrive at \begin{align*}
       \min_{\pi}\max_{(\bm r, \bm v) \in \mathcal{V}} \E[ \Delta (\Nopt(\bm r, \bm v),\pi (D))] \geq  K\left(\frac{1}{2} - \frac{\sum_{(\bm r, \bm v)\in \mathcal{V}}\sum_{(\bm r', \bm v')\in \mathcal{V}} D(\mathbb{P}_{\bm r, \bm v} \lVert \mathbb{P}_{\bm r', \bm v'}) }{C_1 K\lvert \mathcal V \rvert^2} \right).
\end{align*}

In the next section, we recalculate the upper bound of the KL divergences to finish the proof.

\subsection{Upper bounds of KL-divergence}
For any pairs $(\bm r, \bm v),(\bm r', \bm v') \in \mathcal{V},$ we have as in section~5.3,
\begin{align*}
    D(\mathbb{P}_{\bm r, \bm v} \lVert \mathbb{P}_{\bm r', \bm v'}) &= D\left( \prod_{\ell=1}^{4K}\prod_{j = 1}^{\nmin} \mathbb{Q}_{\bm v}(S_{j}^{(\ell)}) \bigg\lVert \prod_{\ell=1}^{4K}\prod_{j = 1}^{\nmin} \mathbb{Q}_{\bm v'}(S_{j}^{(\ell)})  \right)= n_{\min}\sum_{\ell = 1}^{4K} D(\mathbb{Q}_{\bm v} (S^{(\ell)}_1) \rVert \mathbb{Q}_{\bm v'} (S^{(\ell)}_1)).
\end{align*}

Now for each $\ell,$ we show that 
\begin{equation}\label{eq-pairwise-KL-bound-uniform}
D(\mathbb{Q}_{\bm v} (S^{(\ell)}_1) \rVert \mathbb{Q}_{\bm v'} (S^{(\ell)}_1))\leq\begin{cases}
    5K(\epsilon')^2,  & \text{if }  \ell \in (\Nopt(\bm r, \bm v) \cup \Nopt(\bm r', \bm v'))\setminus(\Nopt(\bm r, \bm v)\cap \Nopt(\bm r', \bm v')), \\
    0,& \text{otherwise.}
\end{cases} 
\end{equation}

\begin{proof}[Proof of equation \eqref{eq-pairwise-KL-bound-uniform}]
   Similar as in proof of Lemma~\ref{lem-KL-new-upper-bound-non-uniform}, we denote $\bm q, \bm p$ the mass of distribution induced by $\bm {v}$ and $\bm {v'}$ under $S^{(\ell)}_1$ respectively. 
Then it holds that \begin{align*}
    \sum_{j\in (S^{(\ell)}_1)_+} p_j \log \frac{p_j}{q_j} \leq \sum_{j\in (S^{(\ell)}_1)_+}\frac{p_j(p_j - q_j)}{q_j}= \sum_{j\in (S^{(\ell)}_1)_+}\frac{(p_j - q_j)^2}{q_j}. 
\end{align*} 

Now if $\ell \notin (\Nopt(\bm r, \bm v) \cup \Nopt(\bm r', \bm v'))\setminus(\Nopt(\bm r, \bm v)\cap \Nopt(\bm r', \bm v'))$ we have then $p_i =q_i, \forall i \in (S^{(\ell)}_1)_+,$ that leads to the second inequality.

To prove the first inequality, we divide the upper bound of the last summation into three cases:

\noindent i) For $j = 0$, we have \begin{align*}
    \lvert p_j - q_j \rvert = \lvert \frac{1}{2+{\epsilon'}} - \frac{1}{2}\rvert\leq \frac{{\epsilon'}}{4},
\end{align*}
thus then by $q_j \geq 1/4$, \begin{align*}
    \frac{(p_j -q_j)^2}{q_j}\leq \frac{(\epsilon')^2}{4}
\end{align*}

\noindent ii) For $j \notin \{0,\ell\},$ we have \begin{align*}
    \lvert p_j - q_j\rvert &=  \Big| \frac{1}{(2+{\epsilon'})K} - \frac{1}{2K} \Big|\leq  \frac{{\epsilon'}}{4K},
\end{align*}
thus then by $q_j \geq 1/4K,$
\begin{align*}
    \frac{(p_j - q_j)^2}{q_j} & \leq \frac{(\epsilon')^2}{4K} .
\end{align*}

\noindent iii) For $j = \ell,$ we have \begin{align*}
    \lvert p_j - q_j \rvert = \Big| \frac{1/K+{\epsilon'}}{2 + {\epsilon'}} - \frac{1}{2K}\Big| \leq  \frac{(1+K{\epsilon'})2K - (2+{\epsilon'})K }{4K^2} \leq \frac{{\epsilon'}}{2},
\end{align*}
then $q_j \geq \frac{1}{4K}$ implies \begin{align*}
\frac{(p_j - q_j)^2}{q_j} & \leq  K(\epsilon')^2.
\end{align*}
Combining above three cases, we get then \begin{align*}
\sum_{j\in (S^{(\ell)}_1)_+} p_j \log \frac{p_j}{q_j} & \leq 5K(\epsilon')^2,
\end{align*}
as desired.
\end{proof}

Now with \eqref{eq-pairwise-KL-bound-uniform} and set $\epsilon' =\sqrt{\frac{C_1}{20 K\nmin}},$ we have $\max_{(\bm r, \bm v),(\bm r', \bm v')} D(\mathbb{P}_{\bm r, \bm v} \lVert \mathbb{P}_{\bm r', \bm v'})\leq C_1K/4,$ which then leads to

\begin{align*}
    K\cdot \left(\frac{1}{2} - \frac{\sum_{(\bm r, \bm v)\in \mathcal{V}}\sum_{(\bm r', \bm v')\in \mathcal{V}} D(\mathbb{P}_{\bm r, \bm v} \lVert \mathbb{P}_{\bm r', \bm v'}) }{C_1K\lvert \mathcal V \rvert^2} \right) \geq \frac{K}{4},
\end{align*}
that finishes the proof of \eqref{eq-reduce-to-distance-lb-uniform}, which then implies Theorem~\ref{thm-lower-bound-uniform}.

\end{document}